\documentclass[sigconf,10pt,nonacm]{acmart}
\renewcommand\footnotetextcopyrightpermission[1]{}
\usepackage{booktabs}
\usepackage{enumitem}
\usepackage{amsmath,mathtools}
\usepackage{graphicx}
\usepackage{subcaption}
\usepackage{placeins}
\usepackage{microtype}
\usepackage{xspace}
\usepackage{url}

\newcommand{\eps}{\varepsilon}
\newcommand{\R}{\mathbb{R}}
\newcommand{\E}{\mathbb{E}}
\newcommand{\zhat}{\widehat{z}}

\DeclareMathOperator{\logit}{logit}

\newtheorem{lemma}{Lemma}
\newtheorem{proposition}{Proposition}
\newtheorem{corollary}{Corollary}

\title[Mechanistic Tomography]{Mechanistic Tomography: Designed Measurement for Control-Oriented Interpretability}
\author{Vijay Erramilli}
\email{evijay@gmail.com}

\begin{document}

\begin{abstract}
Mechanistic interpretability seeks quantities that a model does not expose directly: represented states, component effects, interactions, and responses to interventions. Patching, gradients, Hessian-vector products, and subset interventions obtain different measurements under different access assumptions and may target different quantities. We formulate their shared measurement problem as \textbf{mechanistic tomography}: the design and analysis of measurements for recovering internal mechanisms and intervention effects.

For a chosen basis and intervention family, the measurements take the form
\[
\widetilde{y}=Ax+w,
\]
where each row of $A$ describes an intervention, $x$ is the map we want to recover, and $w$ contains nonlinear response, sampling error, and basis misspecification. This gives methods with different targets a common set of questions and a practical sequence: start with the least costly measurements, test on held-out interventions at the intended scale, calibrate a simple mismatch, and expand the family when a structured residual remains. We formalize these decisions through results on measurement error, perturbation design, calibration, and interaction recovery.

The formulation is not specific to control, but control provides a demanding validation setting: once an estimate guides an intervention, it acts as an observer. In a two-HMM belief-state model, control error rises with observer error under a fixed controller and actuator, while target improvement can hide movement of a nuisance state. Under forward-only access, sparse aggregate measurements recover a finite-effect map with fewer interventions than exhaustive coordinate patching. With gradient access, a few finite probes substantially improve a local attribution map, although some held-out gaps remain. Lifted measurements and designed Hessian-vector products recover pair interactions that first-order maps miss, while Tracr shows that the required family depends on the basis. On GPT-2-small IOI, the measurements reproduce conditional backup and identify the Name Mover--Negative Name Mover interaction as the largest held-out predictive term among the three tested cross-group pairs. On Qwen-2.5-7B, a calibrated additive map reaches held-out $R^2=.983$ on a finite refusal-response surface; pairwise lifting gives no detected MAE improvement. The two pretrained-model studies illustrate both branches of the procedure: add interactions when held-out error requires them, and stop at the simpler family when it does not.
\end{abstract}

\maketitle

\section{Introduction}

\begingroup
\emergencystretch=1em
Mechanistic interpretability often asks a simple question: how much does an internal component contribute to a model's output? The model does not report that effect directly, so we have to measure it. Coordinate patching changes one component at a time. Attribution patching uses gradients to estimate many small local effects. Subset interventions change several components together, while Hessian-vector products probe interactions. These methods use different forms of access and may estimate different quantities, but they face the same basic problem: recovering an internal effect from incomplete measurements.
\par
\endgroup

We formulate that problem as $\widetilde y=Ax+w$: each row of $A$ records an intervention, $x$ is the map\footnote{x can be a vector or a matrix depending on which method of interventions were used; we use map in this paper} we want to recover, and $w$ collects what the linear prediction misses. We call the design and analysis of these measurements mechanistic tomography.

Writing the problem this way serves two purposes. The first is structural: it gives methods with different targets a common measurement language. Coordinate patching measures finite one-component effects, attribution patching returns a local gradient map, subset interventions give aggregate forward measurements, and Hessian-vector products measure local curvature~\cite{zhang2024patching,syed2024attrib,bair2026compressed,zhang2026lies}. Their targets remain different. What they share is a set of questions: what is the unknown map, how are the measurements obtained, which interventions do those measurements represent, what does the residual contain, and how will the result be checked? Sec.~\ref{sec:formulation} states this common form and derives the conditions under which its measurements are useful.

The second purpose is operational: the common form enables us to explore how an existing method can do more. For instance, if an effect map is sparse or compressible, aggregate measurements can replace exhaustive one-component probes. A few finite measurements can test whether a local gradient map needs only a simple correction at the intended intervention scale. If a structured error remains on held-out interventions, the next measurement must add interactions, measure curvature, or change the basis. Calibration dimension names the point at which a simple correction stops being enough; it is one decision inside this larger procedure. Sec.~\ref{sec:decisions} develops this logic. It matters in real transformers because backup, suppression, redundancy, and self-repair make the effect of one component depend on the state of others~\cite{wang2022ioi,mcgrath2023hydra,mcdougall2023copy}.

Current mechanistic interpretability benefits from a strong architectural prior: researchers know where to look because people chose the layers, attention heads, residual streams, and token positions. If models increasingly help design their successors, those conventions may become less transparent to the people evaluating them. This strengthens the case for a measurement-first approach whose claims rest on declared interventions and held-out responses rather than on a familiar blueprint.

We call an estimate an \emph{observer} when it guides an intervention. An observer must do more than reconstruct a mechanism: it must predict what an intervention will actually do at the scale where it will be used---a requirement Sec.~\ref{sec:formulation} develops. Control makes observer error visible: it can change the action taken, and successful target control can still move an unwanted state. Sec.~\ref{sec:control} tests both claims while holding the model, controller, and actuator fixed where required. This use of control follows recent work that evaluates interpretability through the interventions it enables and the evidence supporting them~\cite{bhalla2025unifying,orgad2026actionable}.

Mechanistic tomography begins after a candidate basis and observer family have been proposed; Sec.~\ref{sec:limits} explains what the formulation cannot determine on its own.

\subsection*{Contributions}

\begin{enumerate}[leftmargin=*,itemsep=2pt]
    \item \textbf{Structural unification.} Table~\ref{tab:methodmap} places coordinate patching, attribution patching, subset recovery, Hessian-vector products, and lifted recovery in one measurement language. The methods keep their own targets and numerical answers; the framework unifies how their measurements and errors are described.
    \item \textbf{Measurement and extension theory.} Lemma~\ref{lem:measurement-reduction} separates nonlinear response, sampling error, and basis error. The remaining results ask when a target can be recovered, how intervention scale and mask density affect the measurements, and when first-order measurements must give way to interaction-aware measurements.
    \item \textbf{Calibration dimension.} Proposition~\ref{prop:caldim} defines the smallest declared correction family that lets a baseline map predict finite responses on held-out interventions. An $r$-dimensional correction needs $O(r)$ well-conditioned probes; a residual outside that family calls for new features or a new basis. Sec.~\ref{sec:step0} tests the dimension-one case: at the larger intervention scale, one fitted gain raises mean held-out AtP $R^2$ from $.818$ to $.960$, close to sparse recovery at $.969$.
    \item \textbf{Control-facing validation.} In a two-HMM wind tunnel, closed-loop target error tracks observer error across ten observer variants (Spearman $0.95$) under a fixed controller and actuator. A separate specificity test shows why target improvement alone is not enough: it can coexist with increased movement of a nuisance state.
    \item \textbf{Operational extensions.} Secs.~\ref{sec:aggregate}--\ref{sec:qwen} follow the procedure under progressively weaker ground truth. Aggregate measurements replace exhaustive coordinate probes when the map is compressible; finite probes calibrate a local gradient map; lifted and curvature measurements recover interactions; Tracr exposes the role of the basis; and GPT-2-small IOI shows when an interaction term is needed. A held-out Qwen-2.5-7B response surface gives the complementary stopping result: finite calibration makes the additive map adequate, so pairwise lifting is not justified by the measured error.
\end{enumerate}

\section{Mechanistic tomography as an inverse problem}\label{sec:formulation}

How much a model's output changes when we alter one or more components\footnote{A component is any unit that can be intervened on; examples include an attention head, MLP block, sparse auto-encoder (SAE) feature, residual-stream direction, or circuit edge.} depends on how we make the measurement. A finite single-component change, a local gradient, and a multi-component intervention need not return the same value because transformers are nonlinear and component effects can depend on what else changes~\cite{zhang2024patching,syed2024attrib,bair2026compressed}. In this section we unify these different measurement methods under a single inverse problem formulation. 

To formalize mechanistic tomography, we first define what we observe, what we change, and over which inputs we average. Let $c\sim\mathcal{D}$ denote a model context, let $h_c$ denote the activation state at the intervention point, and let $f_c(h)$ be the scalar output metric obtained by continuing the model from activation state h under context c. For example, $f_c$ may be a logit difference between two candidate tokens. For each candidate component $i$, the direction $d_i(c)$ specifies how we intervene on that component, and $D_c=[d_1(c),\ldots,d_n(c)]$ represents the set of interventions. These candidate interventions form the coordinate system in which we will estimate component effects. Choosing that coordinate system is an experimental hypothesis; it does not imply that the model’s computation is inherently organized according to those coordinates.

All of these methods begin with the same physical operation. They choose some components, push the activation along their directions, and record how much the readout moves. A method is defined by which components it changes, how far it moves them, and how it turns the measured responses into a map of component effects.

Before comparing these measurements, we need a convention for turning a choice of components into an actual activation change. Without one, a measurement that touches many components may look stronger simply because it changed more of them. Let $z$ be a mask that activates an expected fraction $\rho$ of the coordinates. We choose a normalization $s_\rho>0$ and write $a=z/s_\rho$ for the coefficients applied to the component directions. For example, $s_\rho=\sqrt{\rho n}$ makes the rows of the design approximately unit norm. The scalar $\eps$ sets the overall size of the push. We divide the observed change by $\eps$ so that measurements taken at different scales are reported in the same units:
\begin{equation}
    \widetilde y(a) = \frac{1}{\eps}\E_c\left[f_c(h_c+\eps D_c a)-f_c(h_c)\right].
    \label{eq:finite}
\end{equation}

Equation~\eqref{eq:finite} gives one observed response to one intervention design. The quantity we want, however, is the underlying effect map $x$: a description of how the chosen components affect the scalar readout. In general, we cannot observe this map directly. A design $a$ gives one scalar view of it, with $a^\top x$ predicting the response and $w(a)$ recording what that prediction misses. In practice, we estimate $\widetilde y(a)$ by applying the same design across many contexts. Repeating that design improves the estimate but does not reveal new parts of $x$. Recovering a multi-component map under forward-only access requires different designs $a_1,\ldots,a_m$; their rows form $A$~\cite{bair2026compressed}. Recovering the hidden map from these designed partial views is the tomography problem.

Together, the measurements give the standard form used in network tomography and compressed sensing~\cite{vardi1996tomo,candes2006robust,firooz2010cs}:
\begin{equation*}
    \widetilde y = A x + w. \tag{MT}
\end{equation*}
For a first-order map, $w$ includes curvature, finite-sample error, and error from the chosen basis. Unlike independent noise in a textbook inverse problem, this residual depends on the design: changing $A$ changes both what we measure and what we do to the model.

This is a measurement model, not a claim that the transformer is linear; Sec.~\ref{sec:decisions} asks when it remains adequate at the intended intervention scale.

When an effect map guides an intervention, it acts as an observer. It must then do more than describe the measurements already made: it must predict the response to a held-out action at the scale where that action will be used. This makes $w$, $\eps$, and $\rho$ part of the measurement design rather than reporting details. Sec.~\ref{sec:control} tests this requirement in a control loop, and Sec.~\ref{sec:discussion} turns it into a practical workflow.

The mapping is easiest to see method by method:

\begin{itemize}[leftmargin=*,itemsep=4pt]
    \item \textbf{Coordinate patching.} A one-component probe gives the row $e_i$. If $x$ is the finite singleton map, then $A=I$ and the measurement reads $x_i$ directly. If $x$ is the local first-order map, the finite response is $\widetilde y(e_i)=x_i+w(e_i,\eps)$. A subset intervention uses a row with several nonzero entries; stacked rows recover the additive map jointly, while nonadditivity, sampling error, and basis error remain in $w$.

    \item \textbf{Hessian measurements.} A Hessian-vector product makes local curvature $H_0$ the target and measures its projections along chosen directions. With forward-only access, the analogous lifted design adds columns $a_i a_j$ and pair coefficients to $x$. Both take the form $\widetilde y=Ax+w$; $w$ contains response beyond the chosen interaction order or basis.

    \item \textbf{Attribution patching.} Shrinking $\eps$ toward zero defines the local first-order map
    \begin{equation}
        x_i = \E_{c\sim\mathcal{D}}\left[\left.\partial_\delta f_c(h_c+\delta d_i(c))\right|_{\delta=0}\right].
        \label{eq:firstorder}
    \end{equation}
    One backward pass returns all coordinates of $x$~\cite{syed2024attrib}. For this local estimand on the measured contexts, $\widetilde y=Ix$ and $w=0$: the inverse problem is degenerate. At finite scale, however, the gradient supplies no information about $w$; Sec.~\ref{sec:step0} uses finite calibration probes to measure it.
\end{itemize}

Table~\ref{tab:methodmap} summarizes this shared measurement structure without implying that the methods use the same basis, estimate the same quantity, or return the same numerical answer.

\begin{table}[t]
\centering
\caption{Mechanistic tomography maps each method to a target, access path, and measurement design.}
\label{tab:methodmap}
\scriptsize
\renewcommand{\arraystretch}{1.08}
\setlength{\tabcolsep}{2pt}
\begin{tabular}{p{0.27\linewidth}p{0.28\linewidth}p{0.35\linewidth}}
\toprule
Method and access & Declared target & Induced measurement \\
\midrule
Coordinate patching; forward & finite singleton effects & one-coordinate rows \\
Subset recovery; forward & additive finite map & multi-coordinate rows; joint fit \\
AtP; backward & local first-order map $x$ & gradient returns coordinates \\
Lifted recovery; forward & finite main and pair effects & columns include $a_i$ and $a_i a_j$ \\
HVP; white-box second order & local curvature $H_0$ & directions give curvature projections \\
\bottomrule
\end{tabular}
\end{table}
\FloatBarrier

\subsection{What the formulation lets us decide}\label{sec:decisions}

The inverse formulation separates three decisions that arise in any measurement-based study of mechanism:
\begin{enumerate}[leftmargin=*,label=\arabic*.]
    \item Does the proposed linear relation describe finite interventions at the scale where the estimate will be used?
    \item If it does, which measurement design identifies the target with acceptable cost and error?
    \item If the relation fails on held-out interventions, can a small calibration repair it, or must the measurement family be enlarged?
\end{enumerate}
These are the natural questions because they are the three stages of an inverse measurement procedure: model the measurements, recover the target, and diagnose the model when it fails.

A map should be tested in the operating region where it will be used, rather than only where it is easiest to measure. For an observer, that region is set by the intervention scale and the contexts the controller will encounter. Held-out finite interventions therefore test whether the map is adequate for action.

Reconstruction and downstream action are separate criteria because an actuator weights errors unevenly. Error outside the actuator's effective directions may have little control cost, while error along those directions reaches the loop directly. Sec.~\ref{sec:result-control} demonstrates this separation; the same distinction appears in network tomography~\cite{roughan2003te}.

The recovery tools come from classical inverse problems, sparse recovery, masked interaction recovery, and HVP correction~\cite{bair2026compressed,kang2025spex,butler2025proxyspex,zhang2026lies}. The new element is that each measurement is also a finite intervention on a nonlinear model, so the design changes the residual it must control.

\noindent\textbullet\quad\textbf{Does one finite intervention produce a linear measurement?}\par
\noindent The inverse formulation would be empty if we simply assumed that every intervention followed it. We instead derive the relation from the response to one finite intervention. The lemma below shows that a mask $a$ produces a weighted sum of the local component effects, $a^\top x$, plus three identifiable sources of error.

\begin{lemma}[Mechanistic measurement reduction]\label{lem:measurement-reduction}
Let $f_c$ be twice differentiable on the perturbation ball. Assume its Hessian satisfies $\|\nabla^2 f_c\|_{op}\le L$. Define the population normalized response
\begin{equation}
    y_\infty(a,\eps)
    =
    \frac{1}{\eps}
    \E_c\!\left[
    f_c(h_c+\eps D_ca)-f_c(h_c)
    \right].
\end{equation}
Suppose the response estimated from a batch of $B$ paired contexts can be written as
\begin{equation}
    \widetilde y_B(a,\eps)
    =
    y_\infty(a,\eps)
    +\frac{e_B(a)}{\eps}
    +\zeta_B(a,\eps),
\end{equation}
where $|e_B(a)|\le\tau_B(a)$ bounds uncancelled error in the response numerator and $|\zeta_B(a,\eps)|\le\nu_B(a)$ bounds the remaining normalized sampling error. Then
\begin{equation}
    \widetilde y_B(a,\eps)
    =
    a^\top x+w(a,\eps),
\end{equation}
where $x$ is the local first-order map in Equation~\eqref{eq:firstorder} and
\begin{equation}
    |w(a,\eps)| \le \frac{L}{2}\eps\E_c\|D_ca\|_2^2
    + \frac{\tau_B(a)}{\eps}+\nu_B(a).
\end{equation}
Stacking measurements from several masks gives $\widetilde y=Ax+w$. The proof is in Appendix~\ref{app:proof-measurement}.
\end{lemma}

A simple example makes the result concrete. Suppose a mask changes only heads 2 and 5, with coefficients $a_2$ and $a_5$. The linear part of the measured response is
\[
    a^\top x=a_2x_2+a_5x_5.
\]
The intervention does not reveal either effect separately. It returns one weighted sum of them. Several masks produce several such equations, and solving those equations recovers the individual effects. This is the same logic used in network tomography and compressed sensing.

The residual $w$ records why the measured transformer response may differ from that weighted sum. Its three terms describe three different problems. The curvature term grows with $\eps$: a larger intervention moves farther from the local tangent and encounters more of the model's nonlinear response. The term $\tau_B/\eps$ grows when $\eps$ becomes too small: any error already present in the response numerator is magnified when we divide by the intervention size. The remaining term $\nu_B$ records normalized sampling error that remains after pairing or batching.

The lemma therefore rules out both extremes. We cannot make the intervention arbitrarily large, because curvature then dominates. We cannot make it arbitrarily small when numerator error remains, because normalization magnifies that error. Pairing the clean and intervened evaluations helps because it can cancel shared variation before division by $\eps$.

The mask affects more than the row of $A$. It also changes the displacement $\|D_ca\|$, the amount of curvature encountered, and potentially the sampling error. The design matrix and the residual are therefore coupled. This is the main difference from a textbook linear inverse problem with independent noise added after the measurements have been chosen.

Lemma~\ref{lem:measurement-reduction} establishes the local measurement relation; it does not establish identifiability, sparsity, or design-independent noise. The results below address those separate questions.

\noindent\textbullet\quad\textbf{Can a small correction repair the finite response?}\par
\noindent For a nondifferentiable endpoint---for example, a behavioral rating, tool outcome, or black-box score---we cannot use the Taylor bound directly. In that setting the finite response is the quantity of interest, and held-out interventions provide the relevant test. The remaining question is whether the difference from the gradient map is a simple scale correction or evidence that the first-order family is missing structure. Calibration dimension formalizes that distinction.

A useful first check is whether all finite effects are approximately a rescaled version of the local map. If so, a few finite probes may be enough to estimate the correction. If not, collecting more probes for the same family will not solve the problem.

\begin{proposition}[Calibration dimension]\label{prop:caldim}
Let $F_\eps(a)$ denote the systematic finite-intervention response on a design family $\mathcal A$, and let $\mathcal Q$ be a declared held-out distribution on that family. Fix a declared feature map $\phi(a)\in\R^q$, baseline $x_0\in\R^q$, and correction dictionary $B\in\R^{q\times r}$. If
\begin{equation}
    F_\eps(a)=\phi(a)^\top(x_0+B\theta_\eps)
    \quad\text{for all } a\in\mathcal A
\end{equation}
and the sampled correction design $\Phi B$ has full column rank, then $\theta_\eps$ is identifiable from $O(r)$ noiseless calibration measurements and stably estimable according to the conditioning of $\Phi B$, where row $j$ of $\Phi$ is $\phi(a_j)^\top$. If no such $\theta$ exists, the residual
\begin{equation}
    \inf_\theta \|F_\eps(\cdot)-\phi(\cdot)^\top(x_0+B\theta)\|_{L_2(\mathcal Q)}
\end{equation}
lower-bounds the error of every observer using that family.
\end{proposition}

For a declared nested family $B_r$ and a tolerance $\delta$ at or above the estimated noise floor, the calibration dimension is the smallest $r$ whose held-out error is at most $\delta$. Dimension zero means that the baseline map transfers without correction. Dimension one allows one scalar gain. Larger values allow several stated corrections, such as separate gains for component groups. If the residual lies outside the entire first-order family, no collection of gains can fix it; the observer needs new features, such as component pairs. Sec.~\ref{sec:step0} tests the one-gain case; Sec.~\ref{sec:interactions} supplies a case in which pair features are necessary.

Appendix~\ref{app:calibration} gives concrete examples and the proof.

This definition gives a practical stopping rule. Start with a gradient map when a backward pass is available; otherwise start with an additive map estimated from forward interventions. Fit the correction on one set of masks and evaluate it on another. The held-out split matters because curvature and correlated masks can make an inadequate family fit its training measurements. In the ideal noiseless case, an $r$-dimensional correction needs only $O(r)$ well-conditioned probes instead of re-estimating $n$ main effects or $O(n^2)$ pairs. Move to a richer family only when the held-out residual remains above the estimated noise floor and the richer family improves it reproducibly.

\noindent\textbullet\quad\textbf{How should scale and mask density be chosen?}\par
\noindent After choosing an observer family, we still need to choose the interventions used to measure it. Two design variables matter immediately. The scale $\eps$ controls how far each probe moves the model, while the density $\rho$ controls how many coordinates each mask changes. Small probes reduce nonlinear error but magnify uncancelled response noise. Dense masks may cover sparse features more efficiently, but they can also move the activation farther or make the inverse problem poorly conditioned.

\begin{corollary}[Noisy first-order scale--density bound]\label{cor:scale-density}
For a mask family $\mathcal A_\rho$ under a fixed actuator normalization, define
\begin{equation}
    \kappa_\rho^2 = \sup_{a\in\mathcal A_\rho}\E_c\|D_ca\|_2^2,
\end{equation}
and let $\tau_B(\rho)$ and $\nu_B(\rho)$ uniformly bound the two sampling terms in Lemma~\ref{lem:measurement-reduction}. If $C_{\mathrm{rec}}(A_\rho)$ bounds the selected recovery method's amplification of measurement error, a conditional first-order error objective is
\begin{equation}
\begin{aligned}
    \mathcal J_1(\eps,\rho)
    &= C_{\mathrm{rec}}(A_\rho)
    \left[
    \frac{L}{2}\eps\kappa_\rho^2
    +\frac{\tau_B(\rho)}{\eps}+\nu_B(\rho)
    \right],\\
    (\eps^*,\rho^*)&\in\arg\min_{\eps,\rho}\mathcal J_1(\eps,\rho).
\end{aligned}
\end{equation}
For fixed $\rho$, if $\tau_B(\rho)>0$ and $C_{\mathrm{rec}},\nu_B$ are locally independent of $\eps$, the first two terms are minimized at
\begin{equation}
    \eps^*(\rho)=\sqrt{\frac{2\tau_B(\rho)}{L\kappa_\rho^2}}.
\end{equation}
Density affects both model-induced error through $\kappa_\rho$ and identifiability through coverage and conditioning.
\end{corollary}

The bound guides what a scale--density study must measure; it is not a post-hoc derivation of the settings used here. Secs.~\ref{sec:aggregate}--\ref{sec:interactions} vary scale, budget, and noise separately rather than estimating $L$, $\tau_B$, or the full $\eps\times\rho$ surface. Sec.~\ref{sec:discussion} returns to this limitation.

Appendix~\ref{app:scale-density} explains the scale and density terms and gives the proof.

\noindent\textbullet\quad\textbf{What does the bound show?}\par
\noindent The square-root dependence of the preferred scale comes from classical noisy forward differences~\cite{shi2022finite}. Mechanistic interventions add two complications. First, the mask family changes the activation displacement as well as the coverage of the unknown components. Second, the same family changes the conditioning of the recovery problem. The term $\tau_B$ represents response error that is present before division by $\eps$ and does not cancel between the clean and intervened evaluations; $\nu_B$ represents the paired sampling error left after normalization. Mask density has no meaning without an actuator convention. Under fixed per-coordinate amplitude, a denser mask usually causes a larger displacement. Under unit-row normalization, it need not. For pair or Hessian targets, some quadratic response becomes signal rather than error, but mask density still controls whether the pair features receive enough independent coverage.

\noindent\textbullet\quad\textbf{What follows from sparse recovery?}\par
\noindent If the component-effect map is sparse or compressible, aggregate masks can recover it with fewer forward evaluations than patching every coordinate. Once the measurement relation holds, basis pursuit and orthogonal matching pursuit (OMP) use the standard sparse-recovery guarantees, provided that the mask design is sufficiently well conditioned~\cite{candes2006robust,tropp2007omp,berinde2008expander}. Coordinate-only forward patching has a simple worst-case limitation: with fewer than $n$ measurements, at least one coordinate remains unmeasured, and the entire effect may be hidden there. This is a minimax statement, not a claim about average model behavior. It also applies only to forward-only recovery; a white-box gradient returns all first-order coordinates in one backward pass.

\noindent\textbullet\quad\textbf{Why do first-order maps miss interactions?}\par
\noindent Sparse recovery cannot recover a term that the measurement model does not contain. This matters for real transformers. Softmax attention, layer normalization, nonlinear MLPs, and downstream heads all allow the effect of one component to depend on the state of another. Backup heads may compensate for an ablated primary head; self-repair can reroute a computation; and copy-suppression heads act conditionally on what earlier heads have written~\cite{wang2022ioi,mcgrath2023hydra,mcdougall2023copy}. An additive model cannot express these dependencies. More first-order measurements only estimate the wrong family more accurately. Pair features are the smallest extension that can represent a two-component interaction.

\begin{proposition}[First-order blindness to pure interactions]\label{prop:blindness}
Let $H_{0,ij}=\partial^2 f/\partial d_i\partial d_j$ at the baseline. If $x_i=x_j=0$ but $H_{0,ij}\ne0$, gradients and singleton first-order measurements do not identify the cross term. At scale $\eps$, the normalized finite response has
\begin{equation}
    \widetilde y(a)=a^\top x+\sum_{i<j}a_i a_j\Gamma_{\eps,ij}+w_\eps(a),
\end{equation}
where $\Gamma_{\eps,ij}=\eps H_{0,ij}+O(\eps^2)$. Lifted features $a_i a_j$ make these finite pair coefficients visible.
\end{proposition}

Appendix~\ref{app:blindness} gives the proof and shows why the failure comes from the measurement family, not the recovery algorithm.

The lifted model treats each product $a_i a_j$ as an additional measurement feature. It can make a pair visible, but it also creates many more columns: $n$ main effects become $n+\binom{n}{2}$ main-and-pair effects. A design that is well conditioned for the original $n$ columns may be badly conditioned after this expansion.

\begin{proposition}[Conditional lifted-recovery reduction]\label{prop:lifted}
Let $\theta_\eps=(x_\eps,\mathrm{vec}\,\Gamma_\eps)\in\R^{N}$ be sparse in a lifted finite main-plus-pair basis, and let $\Phi(A)$ be the corresponding lifted design with rows $[a,\{a_i a_j\}_{i<j}]$. If the lifted misspecification is bounded, sparse recovery gives the same form of stable error bound as classical sparse recovery, conditional on restricted conditioning of $\Phi(A)$ on the relevant support. A first-order restricted-isometry property (RIP) does not imply lifted RIP. Appendix~\ref{app:proof-lifted} gives the reduction to the standard sparse-recovery theorem.
\end{proposition}

With ideal dense independent Rademacher masks, whose entries are independent random $+1$ or $-1$ values, the main-effect columns $a_i$ and pair columns $a_i a_j$ are orthogonal in the population. Real designs are finite and often sparse or constrained. Pair columns can then duplicate or nearly duplicate one another. We must therefore inspect the conditioning of the lifted design and test the recovered map on masks that were not used for fitting.

Appendix~\ref{app:lifted-conditioning} gives a two-component example of this aliasing problem.

Held-out prediction is therefore part of identification, not merely a reporting convention. On the training masks, a true pair and a false pair can produce the same column and appear equally plausible. A separate or deliberately de-aliased mask design tests whether the recovered interaction transfers.

\noindent\textbullet\quad\textbf{How should missing interactions be measured?}\par
\noindent White-box access solves the first-order part cheaply, but one backward pass still does not return the full pair map. We need measurements of curvature as well. A Hessian-vector product (HVP) measures how the gradient changes along a chosen direction without constructing the full Hessian, so designed HVP directions can serve as aggregate measurements of local interactions.

\begin{corollary}[White-box second-order tomography]\label{cor:hvp}
A backward pass gives $\nabla f$, while designed HVP queries give linear measurements of the local curvature map $H_0$. If the response is quadratic on the intervention path, the known relation $\Gamma_\eps=\eps H_0$ converts local curvature to the normalized finite pair map. Otherwise HVPs identify only a local quantity; finite-scale prediction needs path integration or finite calibration. Lifted forward measurements estimate the finite-scale map directly. Integrated Hessians and HVP correction provide precedents~\cite{janizek2020integrated,zhang2026lies}.
\end{corollary}

Appendix~\ref{app:hvp} gives the proof and explains when a local curvature map predicts a finite intervention.

\section{What the formulation cannot decide}\label{sec:limits}

The inverse formulation requires three external specifications that cannot be deduced from model responses alone:

\begin{enumerate}[leftmargin=*,label=\arabic*.]
\item The coordinate basis is an experimental hypothesis. Interventions can target layers, attention heads, MLP blocks, or Sparse Autoencoder (SAE) features. A mechanism that appears sparse in one basis can look distributed in another, and a coordinate that tracks the output may not be causally involved. Because superposition and dictionary learning make coordinate choice non-unique, selecting a basis fundamentally defines the units of the mechanism rather than merely describing them~\cite{arora2014dictionary,elhage2022superposition,leask2025canonical}.

\item Behavioral control does not identify internal mechanisms. An activation direction can steer output behavior without representing the underlying concept---for example, an entropy-aligned direction may serve as an effective steering knob without representing the model's true belief state or causal bottleneck~\cite{agarwal2025scaling}. Control success alone cannot identify an internal circuit; establishing representational faithfulness requires an independent reference, such as an analytic posterior, a compiled Tracr program, or causal interchange tests~\cite{agarwal2025geometry,geiger2024das,geiger2025causal}.

\item Effect maps depend on the context distribution. Because component effects are defined as an expectation over contexts in Equation~\eqref{eq:firstorder}, averaging across diverse prompts can blend distinct mechanisms and accurately describe none of them. An observer's validity is therefore bounded by the specific prompt distribution over which its measurements are collected and evaluated.
\end{enumerate}

Our experiments evaluate these choices through a step-down validation ladder that systematically relaxes ground-truth control. We begin with synthetic HMMs, where exact Bayesian posteriors provide an analytic target for closed-loop control~\cite{agarwal2025geometry}. A planted synthetic harness then fixes both the coordinate basis and active components to validate interaction recovery algorithms. Next, compiled Tracr models keep the underlying algorithm constant while we swap coordinate bases. Finally, we transition to pretrained models: first GPT-2 IOI, which provides documented circuit groups without exact ground truth, and ultimately Qwen-2.5-7B, where only behavioral responses are available. Each step relinquishes one external reference only after the corresponding measurement step has been verified.

\section{Control-facing validation}\label{sec:control}

We begin the empirical evaluation in a fully controlled feedback loop. This check comes first because making internal measurements cheaper is useful only if observer quality affects the downstream action. We run two control checks. First, we hold the transformer, controller, and actuator fixed while varying only the observer---the internal-state estimate that guides the intervention. This isolates whether estimation error produces closed-loop control error. Second, we use an entangled observer to ask whether target improvement can hide movement of an unrelated nuisance state. An observer that mixes target and nuisance information may improve the stated target while changing another representation that the intervention should leave unchanged.\footnote{In a real model, a nuisance state is an internal representation or capability that should remain unchanged under the intervention. Examples include syntactic fluency, reasoning ability, or unrelated factual knowledge when steering a specific concept. In our synthetic benchmark, the nuisance state is the exact posterior log-odds of a second, independent Markov process.}

A control-oriented description separates four objects that steering experiments often merge. The transformer has an internal activation state $x_t$; in classical control terminology, it is the \emph{plant}.\footnote{More precisely, the controlled system includes the transformer state at the edit location and the remaining computation from that state to the readout. We use \emph{transformer} in the text and reserve \emph{plant} for this classical control-theory meaning.} The quantity we want to regulate is $z_t=\phi(x_t)$, but the controller does not observe it directly. It receives measurements $y_t$, uses an observer $O$ to form the estimate $\widehat z_t$, and then applies a control rule $K$:
\begin{align}
    z_t &= \phi(x_t), & y_t &= h(x_t,c_t)+\epsilon_t,\\
    \zhat_t &= O(y_{0:t},c_{0:t}), & u_t &= K(\zhat_t,r_t).
\end{align}
Here $c_t$ denotes the context, $\epsilon_t$ measurement noise, $r_t$ the desired reference, and $u_t$ the intervention. In an activation-steering example, $x_t$ is the residual-stream state, $O$ is a probe or another state estimate, $K$ converts the estimated error into an edit size, and the actuator is the direction along which the activation is changed. The intervention could instead be a logit bias, retrieval insertion, prompt rewrite, or decoding change. Even a well-behaved controller $K$ can fail if the observer $O$ estimates the wrong quantity. We therefore validate the estimate in four ways: agreement with an oracle when one exists, sensitivity under a fixed controller, movement of nuisance variables, and prediction on held-out interventions in a real model.

\begin{figure}[t]
\centering
\includegraphics[width=0.82\linewidth]{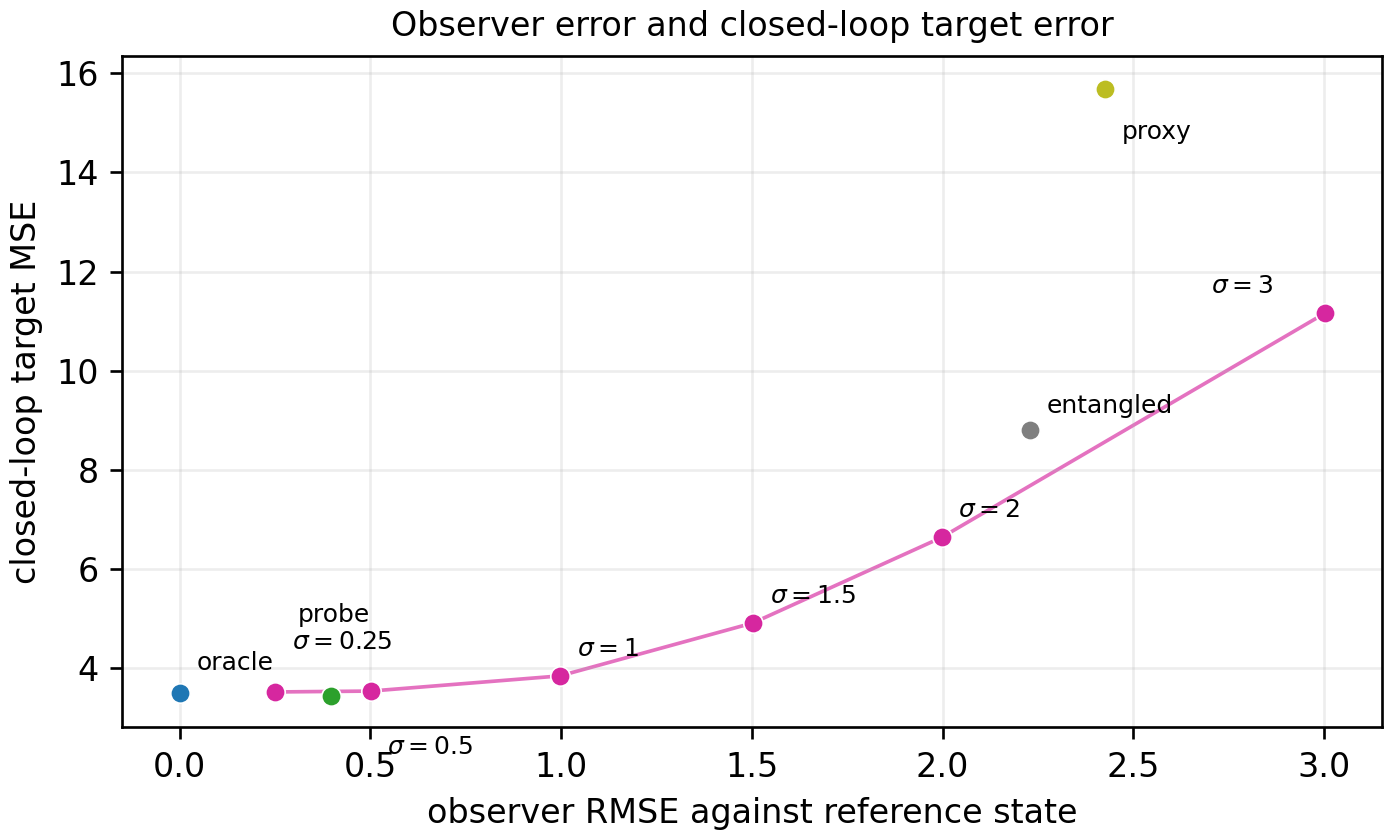}
\caption{With the controller and actuator fixed, closed-loop target error rises with observer error across ten observer variants. Oracle and probe lie low-left; coarse, noisy, proxy, and entangled estimates move up and right.}
\label{fig:observer}
\end{figure}

\subsection{Observer error reaches the loop; target success can hide nuisance movement}\label{sec:result-control}
We need a setting in which the quantity an observer should recover is known exactly. A hidden Markov model (HMM) supplies that reference: after every token, we can compute the exact filtered posterior over its changing hidden state. We can therefore measure observer error before placing the estimate in a feedback loop. This is the reason for starting with an HMM rather than a natural-language task, where the model's true internal belief is not available.

Building on analytic-posterior HMM wind tunnels~\cite{agarwal2025geometry}, we train a small transformer on tokens generated by two independent HMMs. One hidden process defines the target belief $z_1$; the other defines a nuisance belief $z_2$. We use their exact posterior log-odds as the two reference states. The model reaches the Bayes cross-entropy floor, and a linear probe recovers $z_1$ with $R^2\approx0.99$.

The first check asks whether observer quality matters once the estimate enters a feedback loop. We create ten observers with different amounts and kinds of error, including a last-observation proxy that uses only the current emission instead of the filtered history. We then hold the transformer, actuator, controller, and target fixed. Across these observers, closed-loop target error rises with observer RMSE (Spearman $r_s=0.95$). This is a necessary sanity check for the rest of the paper: if observer error did not affect control in a known-state setting, better measurement would have no demonstrated downstream value.

The second check asks a different question. Can an observer improve the target while moving an unrelated state? We deliberately mix the target and nuisance estimates as $z_1+0.5z_2$ and use the corresponding mixed actuation direction $d\propto d_1+0.5d_2$. This condition reduces target MSE from $25.02$ without control to $4.81$, while the oracle target-only condition reaches $3.56$. Yet nuisance movement rises from $0.037$ to $0.079$. Because this stress test changes both the observer and the actuator, it does not isolate observer error. It establishes the narrower point for which it was designed: target improvement alone does not show that the control action used the intended internal direction.

The two criteria mostly agree here, and where they do not, the reason is instructive. The linear probe sometimes produces slightly lower closed-loop error than the analytic posterior oracle, even though the oracle is the more faithful estimate of the external HMM posterior. The controller edits the state represented by the transformer, and the probe can align more closely with that editable state than the external posterior does. This is a measured separation between reconstruction and control utility. We therefore report both: the analytic posterior tests representational faithfulness, while closed-loop performance tests how well an observer matches the state used by this model--actuator pair.

\section{Operational extensions}\label{sec:extensions}

Having established that observer quality affects downstream control, we now ask how cheaply an adequate observer can be measured under forward-only and white-box access.

\subsection{Aggregate measurements replace exhaustive probes when the map is compressible}\label{sec:aggregate}

Without gradient access, coordinate patching requires a separate forward intervention for every coordinate. Mechanistic tomography suggests another option: change several coordinates together and treat the response as one aggregate measurement. Each such intervention supplies one row of $A$. If only a few coordinates have large effects, several aggregate rows may recover the map with fewer interventions than probing every coordinate separately.

We test this possibility in the same HMM model used for the control experiment. We define 32 coordinates by crossing transformer layer with token position. This dimension is small enough that we can also patch all 32 coordinates individually. Their finite effects provide a reference against which to evaluate the aggregate recovery.

The reference map is not exactly sparse, but it is strongly compressible: its four largest coordinates contain $98.4\%$ of its squared $\ell_2$ norm. We therefore use orthogonal matching pursuit (OMP), a sparse-recovery method that selects a small set of coordinates and fits their effects from the aggregate responses. A separate validation split chooses how many coordinates OMP should retain. We also fit ridge regression as a dense baseline.

After 12 aggregate measurements, validation selects a four-coordinate OMP model. It reaches Pearson $r=0.989$ against the coordinate-patching reference and $R^2=0.935$ on held-out aggregate masks. Ridge regression requires 32 measurements, the same count as exhaustive coordinate patching, to reach $r=0.962$ and held-out $R^2=0.933$.

This is a positive but conditional result. Aggregate measurements save interventions here because the finite-effect map is compressible and the aggregate response is approximately additive. They would not provide the same advantage for a dense map or one dominated by interactions. Bair et al. demonstrate the same sparse-recovery principle at pretrained-model scale using attention-head knockouts~\cite{bair2026compressed}. Here the experiment establishes the forward-only baseline before we consider the cheaper access provided by gradients.

\begin{figure}[t]
\centering
\includegraphics[width=0.82\linewidth]{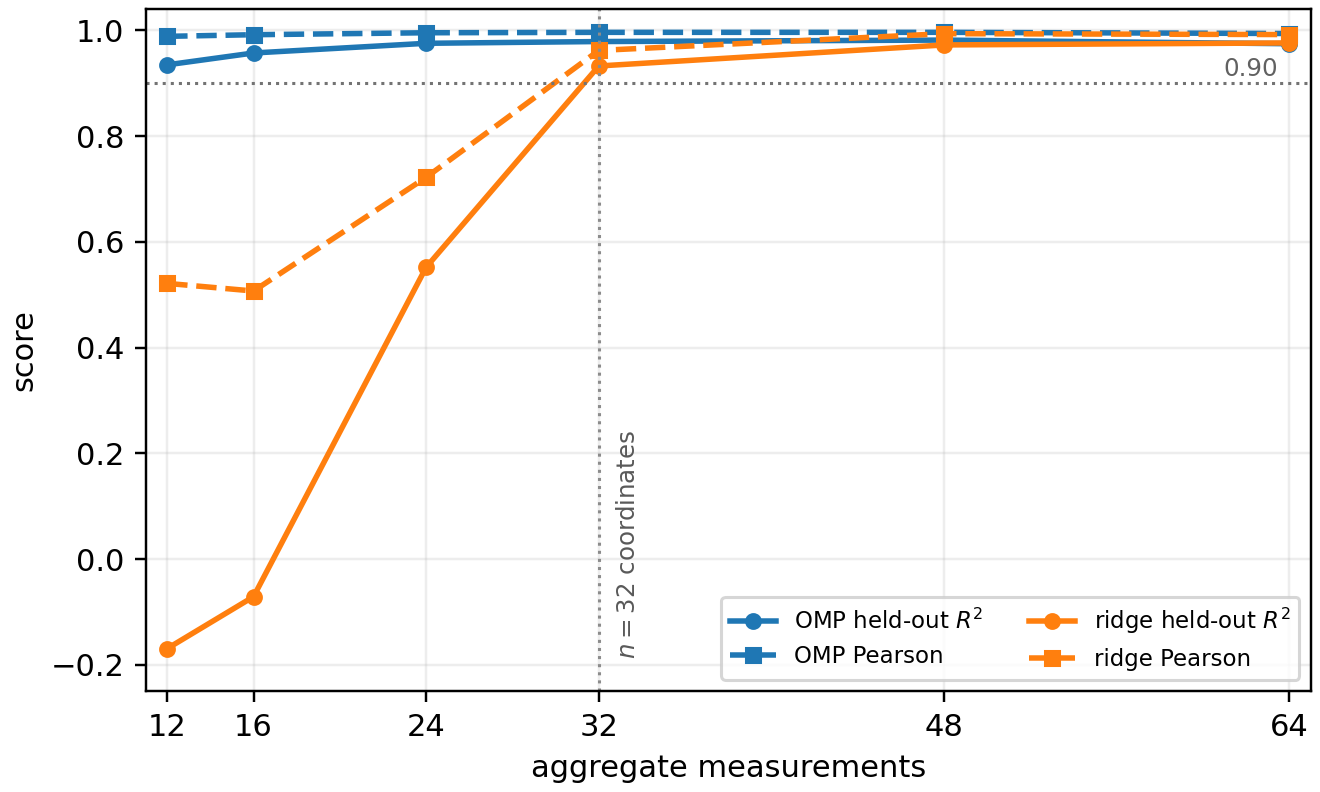}
\caption{Aggregate recovery of the 32-coordinate finite-effect map. OMP scores Pearson $r=0.989$ and held-out $R^2=0.935$ with 12 aggregate measurements. Ridge reaches comparable held-out prediction with 32 measurements, the exhaustive coordinate count.}
\label{fig:forward}
\end{figure}

\subsection{A few finite probes calibrate attribution patching at intervention scale}\label{sec:step0}

With gradient access, attribution patching returns the complete local first-order map in one backward pass. This is cheaper than recovering the same map from many forward measurements. The limitation is that a gradient describes an infinitesimal change, while an actual intervention has finite size. Curvature and saturation can make the finite response smaller than, or otherwise different from, the gradient prediction.

We therefore ask the cheapest question first: is the mismatch mostly a single scale factor, or is the local map missing structure? We test two intervention scales, $\eps\in\{5,8\}$, using the same trained HMM model and component directions as in Sec.~\ref{sec:aggregate}. Raw AtP already predicts the smaller intervention well. At $\eps=8$, saturation makes the finite response smaller than the gradient extrapolation, so raw AtP systematically overpredicts it. We fit one scalar gain from a set of finite calibration probes and apply it to the entire gradient map. This correction changes the scale of the map without re-estimating its coordinates.

A gain can look successful simply because it was fitted on a favorable set of sequences or masks. We therefore cross three independently generated evaluation batches with three independently generated mask-and-split designs, giving nine evaluation cells. Each cell contains 128 masks. We use 72 to fit the scalar gain and the sparse OMP comparison, 24 to select the OMP model, and 32 only for final evaluation. The gain may depend on intervention scale; we do not treat it as a universal property of the model.

\begin{table}[t]
\centering
\caption{Finite-scale calibration on one trained model whose intervention directions are estimated once. We cross three independently generated 256-sequence evaluation batches with three mask-and-split designs. Entries are means [min, max] across the nine cells.}
\label{tab:step0}
\scriptsize
\setlength{\tabcolsep}{3pt}
\begin{tabular}{lrrrr}
\toprule
Scale & raw AtP $R^2$ & calibrated AtP $R^2$ & OMP $R^2$ & fitted gain $\widehat g$ \\
\midrule
$\eps=5$ & 0.943 [0.920, 0.967] & 0.963 [0.933, 0.983] & 0.965 [0.952, 0.976] & 0.874 [0.799, 0.915] \\
$\eps=8$ & 0.818 [0.751, 0.903] & 0.960 [0.925, 0.978] & 0.969 [0.960, 0.977] & 0.742 [0.677, 0.780] \\
\bottomrule
\end{tabular}
\end{table}

At $\eps=8$, calibration raises mean held-out $R^2$ from $0.818$ to $0.960$, close to OMP's $0.969$. A gain fitted on the other two evaluation batches gives a leave-one-batch-out mean of $0.959$ across the nine target/design cells. The correction therefore transfers across sequence batches rather than merely fitting one set of examples. At $\eps=5$, where raw AtP is already strong, calibration raises mean held-out $R^2$ from $0.943$ to $0.963$, compared with $0.965$ for OMP.

We next ask how many finite probes the correction needs. From the 96 masks not reserved for final evaluation, we repeatedly draw calibration sets of 4, 8, and 16 probes; the 96-probe condition uses the full pool. Eight probes already give mean held-out $R^2=0.951$ at $\eps=5$ and $0.948$ at $\eps=8$. Sixteen probes raise these values to $0.957$ and $0.955$. This is why calibration is the first extension to try when gradients are available: a handful of finite measurements can correct most of the scale error without recovering the component map again.

The result is positive, but it is not a blanket endorsement of scalar calibration. A scalar correction is the calibration-dimension-one hypothesis. Even after all 96 probes, some evaluation-batch/design cells retain a gap to OMP. The data therefore support the scalar gain as a useful approximation, not as an exact description of the finite map. Held-out error supplies the stopping rule: stop when the corrected map predicts the intended interventions well enough; otherwise move to a richer observer.

\begin{figure}[t]
\centering
\begin{subfigure}{0.49\linewidth}
\includegraphics[width=\linewidth]{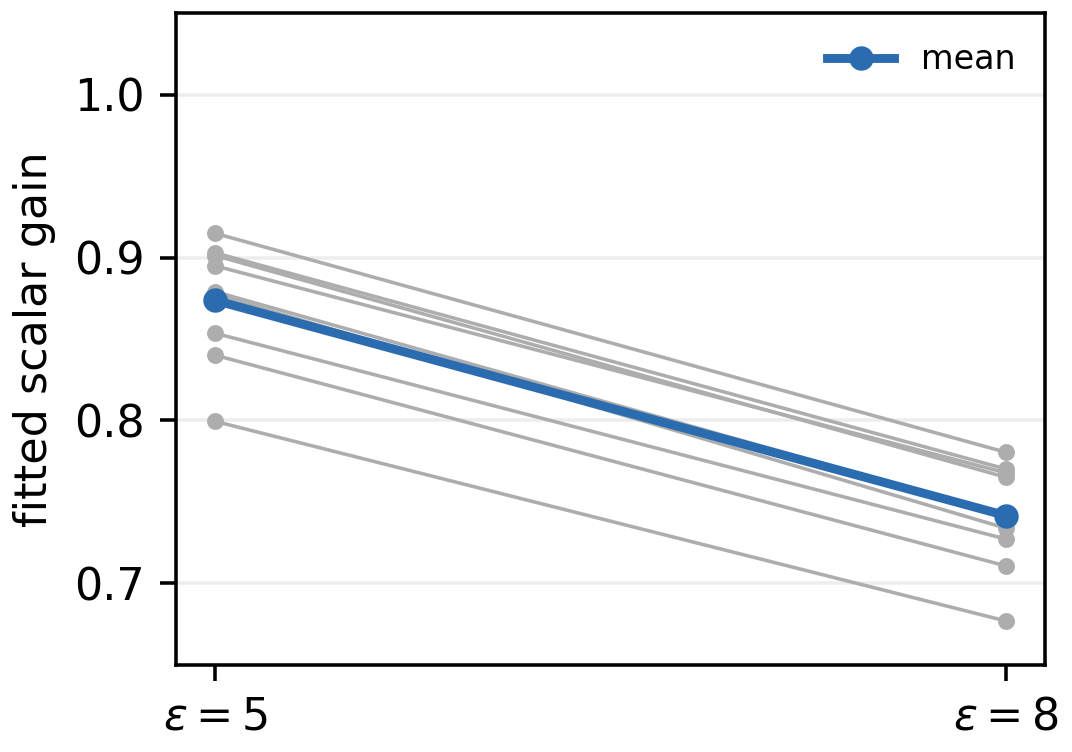}
\caption{Gain shift with scale}
\end{subfigure}\hfill
\begin{subfigure}{0.49\linewidth}
\includegraphics[width=\linewidth]{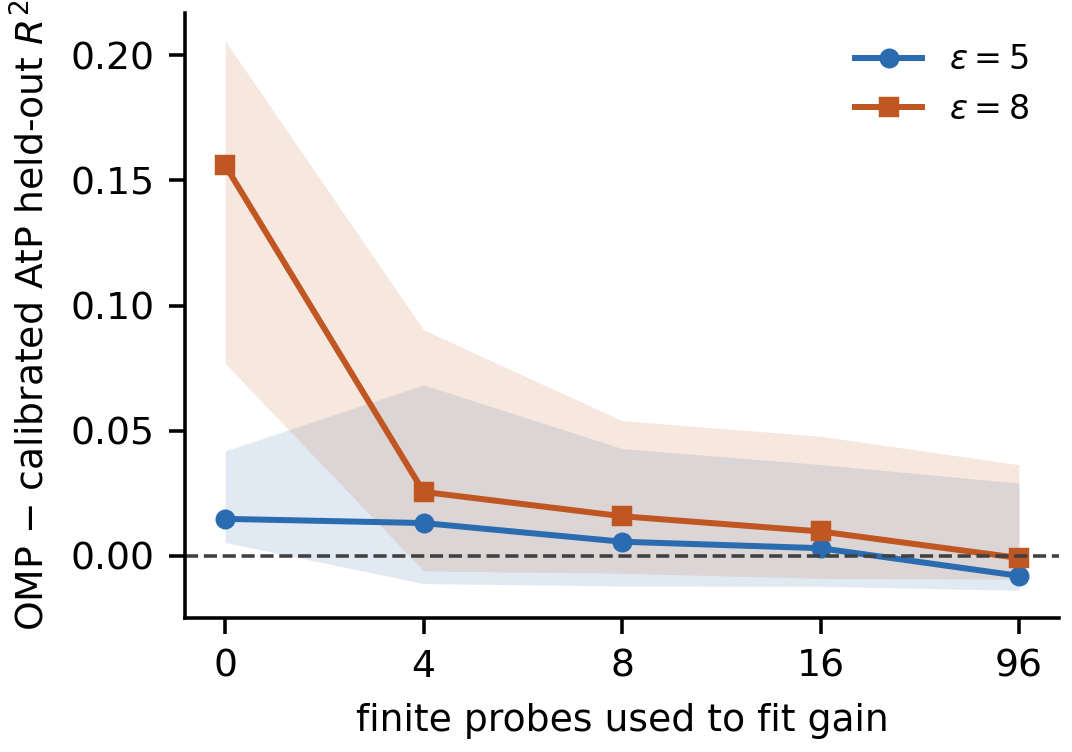}
\caption{Small-budget calibration}
\end{subfigure}
\caption{Finite-scale calibration across the fixed-direction $3\times3$ design. Left: fitted gain at each scale for the nine evaluation-batch/design cells. Right: the median held-out $R^2$ gap between OMP and calibrated AtP as the calibration budget grows; bands show the 10th--90th percentiles over cells and subset draws.}
\label{fig:step0}
\end{figure}

\FloatBarrier
\subsection{Interaction measurements recover effects that first-order maps miss}\label{sec:interactions}

Scalar calibration can change the size of a first-order effect, but it cannot create an effect that appears only when two components change together. If the response contains a term such as $a_i a_j$, no gain applied to an additive map can represent it. This matters for transformers because backup, suppression, redundancy, and self-repair make the effect of one component depend on whether another is active. Before studying these effects in a pretrained model, we use a planted system in which the correct terms are known. This lets us distinguish three failures that can otherwise look alike: choosing a family that cannot represent the response, collecting too few measurements, and selecting a correlated but noncausal feature.

The harness contains 64 candidate components. We plant four nonzero terms: two ordinary main effects, one pure interaction with no first-order signal, and one pair that models redundancy or self-repair. We can also add an off-path confound. Designed-mask methods such as SPEX and ProxySPEX use the same general principle to recover sparse feature and attention-head interactions~\cite{kang2025spex,butler2025proxyspex}.

To represent pair effects, we expand each measurement row. In addition to the original coefficients $a_i$, the row now contains every product $a_i a_j$. The unknown map therefore has one coefficient for each component and one for each unordered pair:
\begin{equation}
    N = n + \binom{n}{2} = 64 + 2016 = 2080.
\end{equation}
We call this the \emph{lifted} basis because pair products become ordinary columns in a larger linear model. Only four of the 2,080 possible terms are nonzero. Measuring every pair separately would require 2,016 pair probes. Aggregate masks may recover the same sparse map with far fewer measurements.

We ask four questions in turn: whether the family can represent the response, how many measurements recovery needs, whether white-box access offers a cheaper route, and how noise and aliases affect the answer.

\paragraph{Can the family represent the response?}

We begin with a generous budget of 256 low-noise aggregate measurements. This removes measurement scarcity as a likely explanation for failure. At $\eps=8$, lifted OMP recovers the planted pair structure and gives nearly perfect prediction on held-out masks. Raw AtP, scalar- and multi-gain AtP, finite singleton patching, and first-order OMP all omit pair features. Every one remains below $0.5$ held-out $R^2$. The gap is representational rather than a fitting failure: the additive families do not contain the planted cross terms, so additional calibration measurements cannot repair them.

\begin{figure}[t]
\centering
\includegraphics[width=0.92\linewidth]{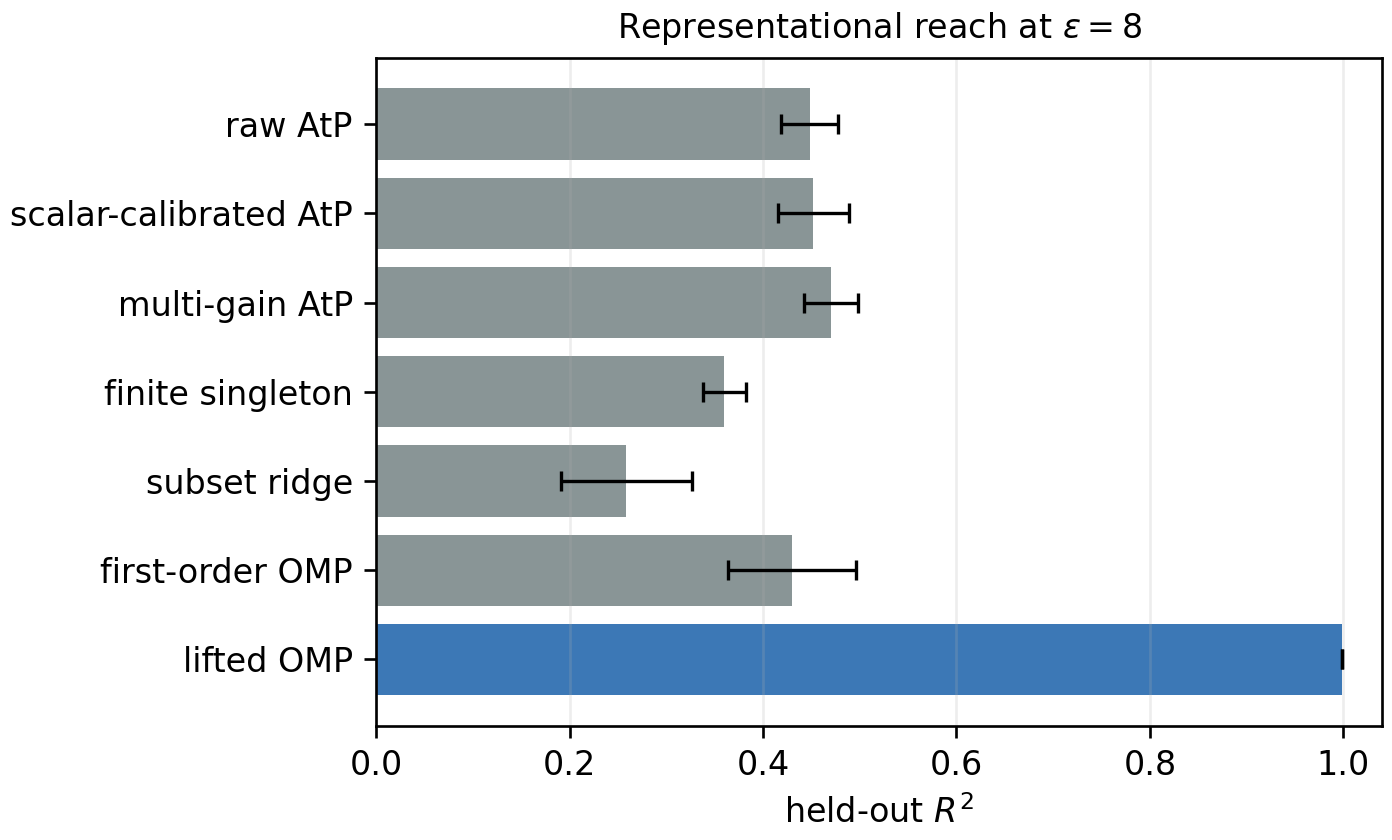}
\caption{Interaction-aware recovery at $\eps=8$ with a generous fixed budget. Lifted OMP recovers the planted pair-bearing response, while methods without pair features remain below $0.5$ held-out $R^2$.}
\label{fig:claim3bar}
\end{figure}

\paragraph{How many forward measurements are needed?}

Once we know that the lifted family can represent the response, we reduce the measurement budget. We require both held-out $R^2\ge0.95$ for finite-response prediction and pair recall $\ge0.99$ for the planted support. Lifted OMP first meets both conditions with 64 aggregate measurements at $\eps=5$ and 72 at $\eps=8$. Exhaustively measuring every pair would require 2,016 probes.

The usual sparse-recovery heuristic gives $4\log(N/4)\approx25$ measurements for four nonzero terms in 2,080 dimensions. We report this only for orientation. It omits constants, mask normalization, finite-sample variation, and the conditioning of the lifted design. Low-budget trials are also variable. We therefore report the first budget at which both prediction and support-recovery thresholds hold, rather than selecting the best isolated run.

\begin{figure*}[t]
\centering
\includegraphics[width=0.92\textwidth]{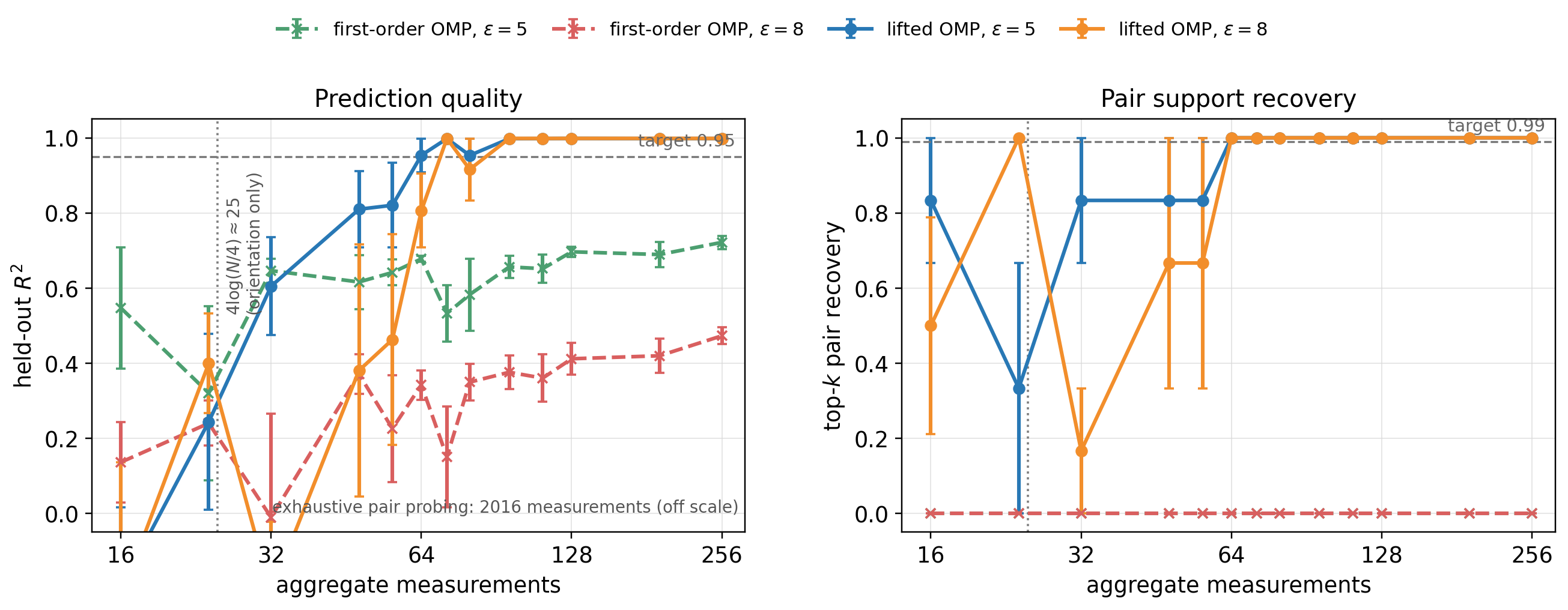}
\caption{Measurement budget for lifted recovery. With four planted nonzero terms among 2,080 main-and-pair features, lifted OMP reaches the prediction and pair-recall thresholds with 64--72 aggregate measurements, compared with 2,016 exhaustive pair probes. First-order OMP has no pair features.}
\label{fig:budget}
\end{figure*}

\paragraph{White-box alternative: AtP plus HVPs.}

Lifted forward measurements are useful when the endpoint is available only through model evaluations, as with sampling, decoding, retrieval, tool use, or a behavioral score. With white-box access, we can measure the local pair structure more directly. A gradient asks how the output initially changes along each component direction. An HVP asks how those gradient effects change when we move in a chosen direction. Attribution patching supplies the main-effect map in one backward pass, while designed HVP queries provide aggregate measurements of pair curvature.

The planted response is exactly quadratic. For this system, multiplying the local HVP pair coefficient by $\eps$ gives the normalized finite-scale pair coefficient. We know this conversion because we constructed the response function. We do not assume that the same conversion holds along a finite intervention path in a transformer.

The combined AtP--HVP observer reaches held-out $R^2\ge0.95$ and pair recall $\ge0.99$ with 12 HVP queries at $\eps=5$ and 24 at $\eps=8$. These are empirical thresholds for this harness, not a scaling law. They count the HVP queries; the full method also requires the AtP backward pass that supplies the main effects.

The pair-only HVP ablation finds the correct interacting pairs but reaches only $R^2=0.36$ at $\eps=5$ and $0.60$ at $\eps=8$. Recovering the pair support is therefore not enough to predict the full response. The observer must combine the AtP main effects with the HVP pair effects. Replacing the gradient main effects with finite singleton effects also performs worse. In this harness, singleton interventions saturate and underpredict larger aggregate changes, while the gradient remains the better main-effect approximation once the missing pairs are included.

\begin{figure*}[t]
\centering
\begin{subfigure}{0.49\textwidth}
\includegraphics[width=\linewidth]{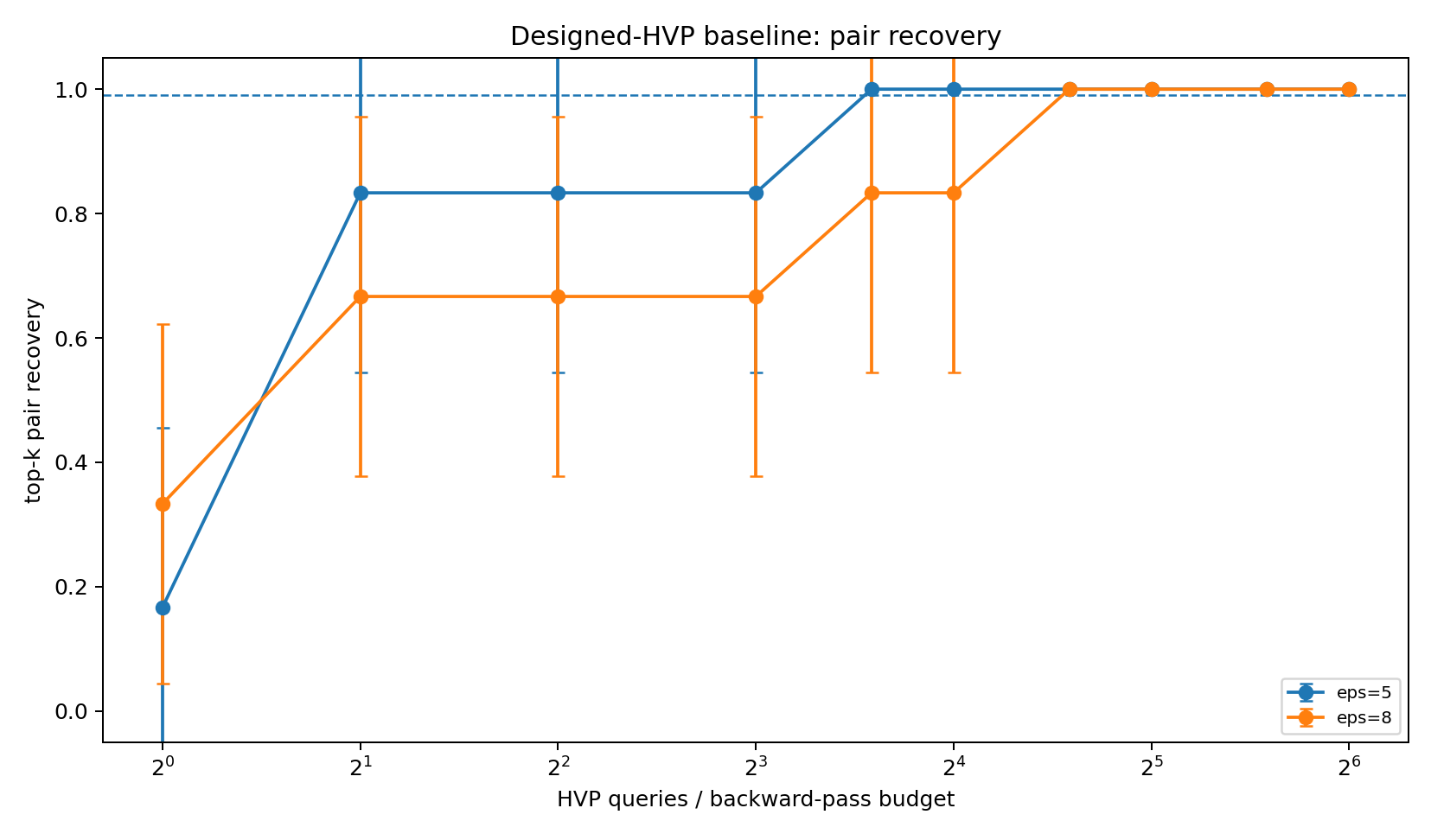}
\caption{pair recovery}
\end{subfigure}\hfill
\begin{subfigure}{0.49\textwidth}
\includegraphics[width=\linewidth]{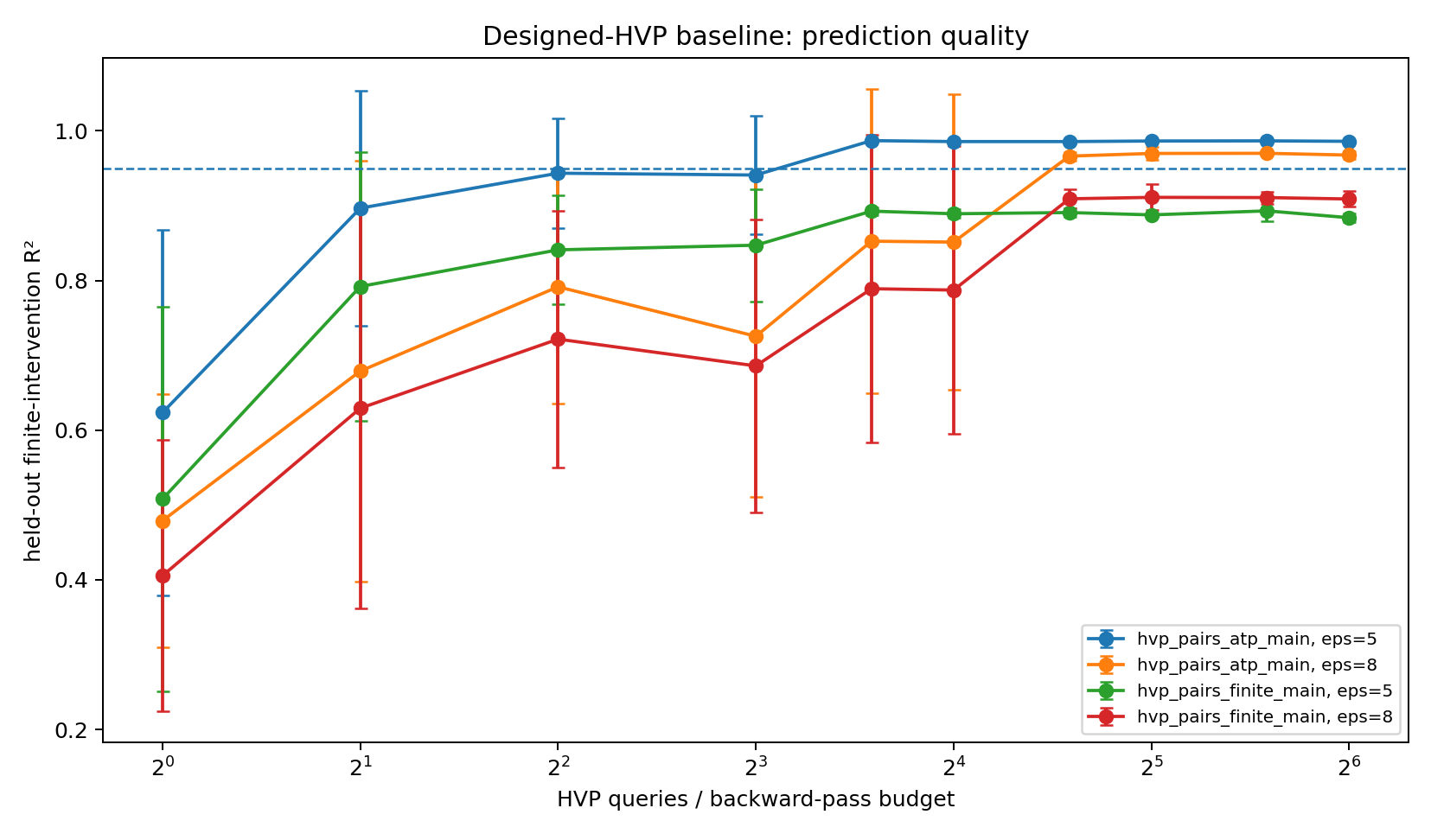}
\caption{held-out prediction}
\end{subfigure}
\caption{White-box interaction recovery. Designed HVPs recover the local pair map, and the quadratic harness supplies the known conversion to finite-scale pair coefficients. Accurate prediction also requires the AtP main-effect map.}
\label{fig:hvp}
\end{figure*}

\paragraph{Support recovery and response prediction are different.}

We next use 96 forward measurements, above the clean recovery threshold, and add zero-mean Gaussian noise to each normalized response. This separates two possible goals: locating the interacting pairs and estimating their coefficients accurately enough to predict a new intervention.

At $\eps=5$, both prediction and pair recall remain high through noise standard deviation $0.01$, while pair recall alone persists through $0.05$. At $\eps=8$, both remain high through $0.05$, and recall alone persists through $0.2$. The interacting pair can therefore remain identifiable after coefficient error has made the observer unsuitable for finite-response prediction.

\begin{figure*}[t]
\centering
\includegraphics[width=0.92\textwidth]{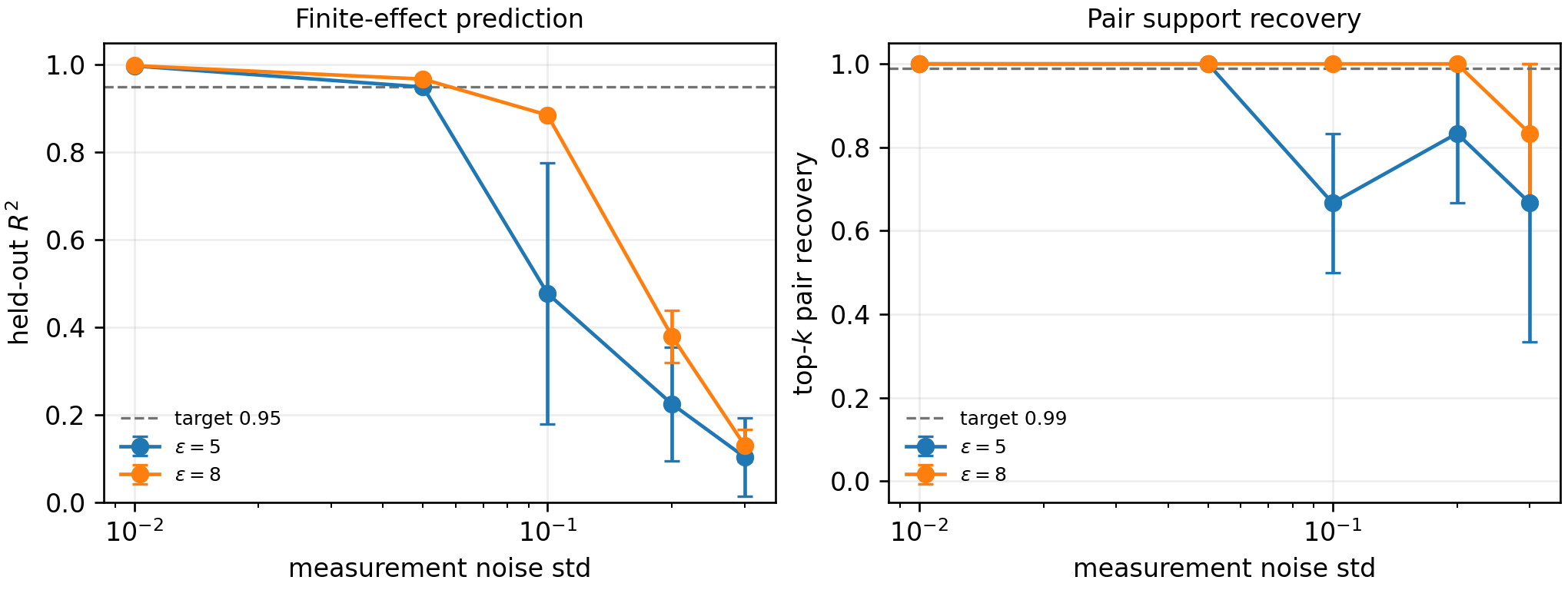}
\caption{Noise robustness with 96 aggregate measurements. Pair support can remain recoverable after held-out finite-response prediction degrades, separating support identification from coefficient accuracy.}
\label{fig:noise}
\end{figure*}

\paragraph{Why held-out masks matter.}

Finally, we give one off-path coordinate the same pattern as a causal coordinate on the training masks. Either feature can then explain the observed responses, even though only one belongs to the planted mechanism. This is a measurement alias.

Sparse support selection followed by independent held-out masks suppresses the false positive. The test shows why held-out interventions participate in identification rather than serving only as a final score. It does not exhaust the problem: a stronger future stress test should alias a false pair with a true pair inside the lifted design.

The planted harness shows when interaction-aware measurements are necessary and how forward-only and white-box access lead to different recovery routes.

\subsection{The basis determines the required measurement family}\label{sec:basis}

The planted experiment in Sec.~\ref{sec:interactions} gives us the correct coordinates in advance. Real interpretability work does not. The same computation can look additive in a basis that contains a completed detector and interactional in a basis that contains only the detector's inputs. We therefore need a system in which we can change the basis while keeping the underlying computation known.

A simple conjunction illustrates the issue. Suppose the model must report whether the previous token is A and the current token is B. If the basis contains one completed ``AB detected'' coordinate, an additive readout can use that coordinate directly. If the basis contains only separate coordinates for ``previous token is A'' and ``current token is B,'' neither coordinate is sufficient by itself. Their product is required. The computation has not changed; only the coordinates used to describe it have changed.

Tracr provides this intermediate setting~\cite{lindner2023tracr}. It compiles programs written in the restricted-access sequence processing (RASP) language into executable transformers. The compiler also assigns names to residual-stream coordinates associated with program variables. We can therefore inspect the same known computation through different sets of coordinates.

We compile three programs over the symbols A, B, and C. The first reports a property of the current token and serves as an additive control. The other two detect the sequence pattern ``A followed by B.'' One uses a numeric representation and the other a boolean representation. Both compute the same conjunction but place it in different compiler-generated coordinates.

At sequence length 6, we enumerate all 243 possible input sequences. After excluding the position with no predecessor, this gives 1,215 position-level examples. We compare two bases. The \emph{native} basis contains the full set of compiler variables, including the completed AB detector. The \emph{restricted} basis contains only variables available before the conjunction has been formed, such as the current-token and previous-token labels.

\paragraph{Does the apparent interaction order change with the basis?}

For the current-token control, an additive model is sufficient in either basis. Both the full native basis and the token-only basis reach held-out position-level $R^2=1.0$.

The AB detectors are also additive in the native basis. This does not mean that the conjunction has disappeared. The compiler has already created an explicit coordinate that records whether ``A followed by B'' is true. Once that completed detector is available as one coordinate, an additive readout can use it directly.

The result changes in the restricted basis. This basis contains the current-token and previous-token information but no completed AB detector. An additive model explains only about half of the detector variance. Adding pair products restores held-out $R^2=1.0$ and selects the intended term:
\[
\text{previous-token-A}\times\text{current-token-B}.
\]
The underlying program has not changed. Only the coordinates have changed. The AB computation is first-order in the native basis and interactional in the restricted pre-conjunction basis.

\begin{figure}[t]
\centering
\includegraphics[width=0.95\linewidth]{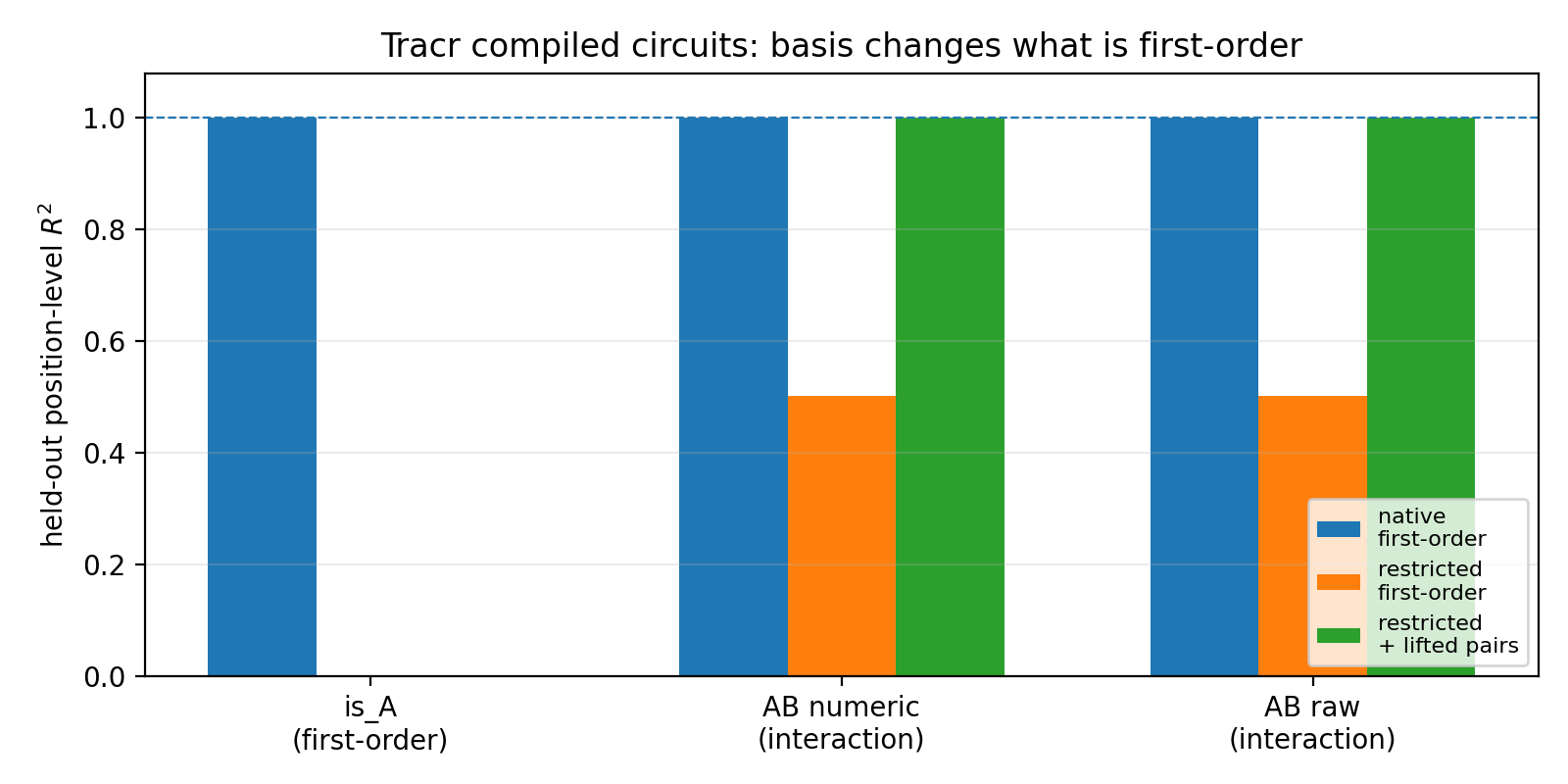}
\caption{Basis comparison in Tracr circuits. The current-token control is additive in both bases. The AB detectors are additive in the native compiler basis because it contains an explicit detector coordinate. In the restricted pre-conjunction basis, additive features explain only about half the variance, while lifted pair features recover the intended conjunction and restore held-out $R^2=1.0$.}
\label{fig:tracr}
\end{figure}

\paragraph{Does the readout use the reconstructed detector?}

Predicting a compiler label does not show that the model's output depends on it. We therefore intervene in the final residual stream, immediately before the readout. The numeric detector group contains the detector and its AB/OTHER coordinates. The boolean group contains the True/False detector and map coordinates.

We first ablate the detector group. We then reconstruct it from the restricted pre-conjunction basis using either additive features or lifted pair features and write that reconstruction back into the residual stream. This tests whether the recovered detector supplies information that the final readout actually uses.

The intervention location is deliberate. By editing the final residual stream, we isolate the connection between the detector subspace and the readout. We do not claim that the same reconstruction would survive the remaining transformer computation if it were written into an earlier layer.

Ablating the detector group drops held-out readout $R^2$ from $1.00$ to about $-0.05$ for the numeric program and $-0.51$ for the boolean program. The final output therefore depends on these detector coordinates. Writing back the additive reconstruction restores readout $R^2$ to only about $0.55$ for both programs. Writing back the lifted reconstruction restores the original $R^2=1.00$.

The same pattern appears when we score reconstruction of the detector group itself. Additive features reach about $.68$ for the numeric group and $0.93$ for the boolean group. Lifted features reconstruct both groups perfectly. The apparently high additive score for the boolean group is misleading: it mostly captures the common negative case and misses the rare positive conjunction that matters to the output.

The negative $R^2$ after boolean ablation also has a simple explanation. The boolean group contains complementary True and False coordinates. Setting both to zero creates a state that never occurs during normal execution. The readout must extrapolate from this out-of-distribution state and performs worse than a constant mean predictor. This is an ablation artifact, not evidence that the detector is harmful.

\begin{figure*}[t]
\centering
\includegraphics[width=0.98\textwidth]{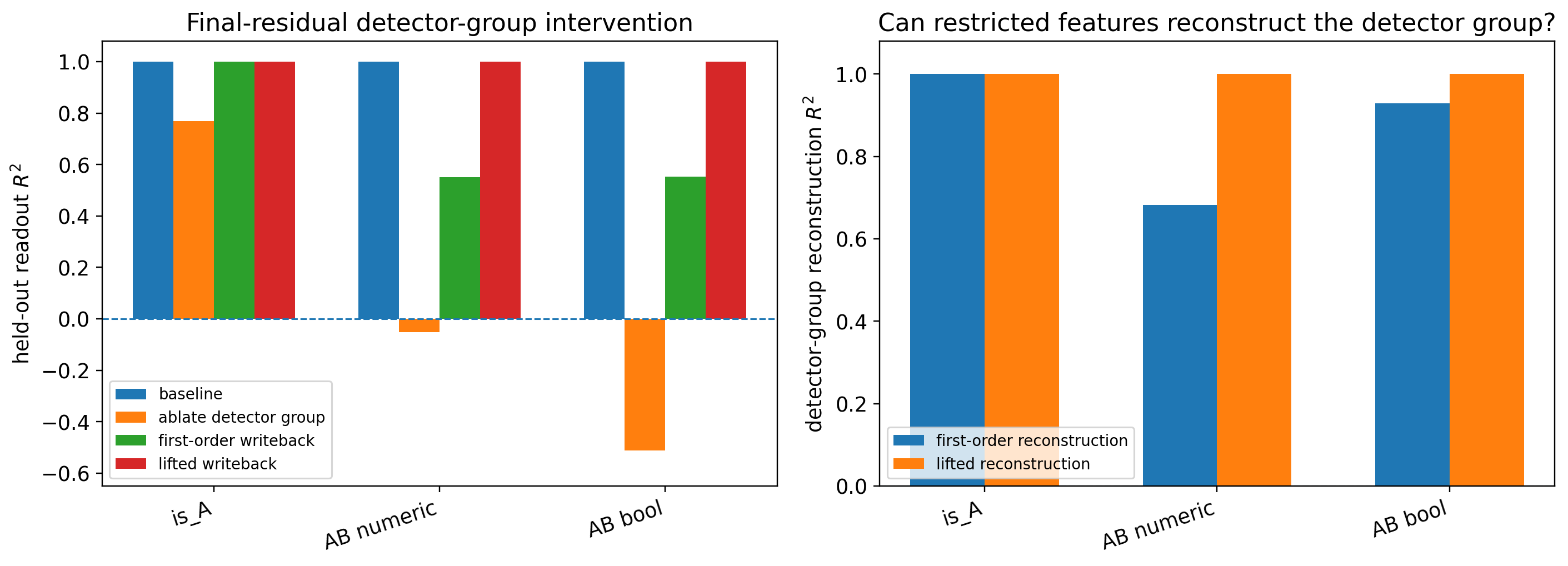}
\caption{Detector-group writeback in Tracr circuits. Left: ablating the final-residual detector group destroys output prediction for the AB programs, while writing back the lifted reconstruction restores the original held-out $R^2$. Right: additive features only partially reconstruct the detector groups, whereas lifted features reconstruct them completely.}
\label{fig:tracr_group}
\end{figure*}

The conclusion is direct. The AB computation is additive in a basis that contains its completed detector, interactional in a basis that contains only its inputs, and additive again after those inputs are lifted with pair features. The writeback test shows that the recovered detector subspace is used by the final readout.

\subsection{IOI reproduces conditional backup and ranks the primary--negative interaction first}\label{sec:ioi}

Tracr lets us check an observer against compiler labels and a known program. A pretrained language model removes that ground truth. We therefore end with the indirect-object identification (IOI) task in GPT-2-small. Prior work has identified a circuit for this task and documented several qualitative behaviors, including a conditional backup response~\cite{wang2022ioi,mcdougall2023copy}. This gives us a useful intermediate case: we know which groups to examine and roughly what they do, but we do not know the finite effect of every possible intervention.

Consider the prompt ``When Alice and Bob went to the store, Alice gave a book to \ldots''. The model should predict Bob, the indirect object (IO), rather than Alice, the repeated subject (S). Three documented head groups contribute to this choice. Primary Name Movers write evidence for Bob toward the output. Backup Name Movers provide another route for writing Bob and become important when the primary route is weakened. Negative Name Movers are different: they write a negative signal against names that earlier components have already begun to copy. They act as a copy-suppression or calibration mechanism rather than as another answer-writing path~\cite{mcdougall2023copy}. On the IOI intervention surface studied here, their net effect lowers the IO-versus-subject readout, so removing them can improve that particular readout.

Let $P$ denote the three primary Name Mover heads $(9.9,\allowbreak 9.6,\allowbreak 10.0)$, $B$ the eight Backup Name Movers $(9.0,\allowbreak 9.7,\allowbreak 10.1,\allowbreak 10.2,\allowbreak 10.6,\allowbreak 10.10,\allowbreak 11.2,\allowbreak 11.9)$, and $E$ the two Negative Name Movers $(10.7,\allowbreak 11.10)$. A label such as $9.9$ means layer 9, head 9. We measure the final-token margin
\[
M=\logit(\mathrm{IO})-\logit(\mathrm{S}).
\]
In the example, $M>0$ means that Bob ranks above Alice, while $M<0$ means that Alice ranks above Bob. A change in $M$ does not necessarily change the emitted token: the output flips between these two names only if the margin crosses zero, and another vocabulary token could also rank above both.

For an ablated set of heads $U$, let $\Delta(U)$ be the clean margin minus the ablated margin. Positive $\Delta(U)$ means that the ablation moves the model away from Bob and toward Alice. Negative $\Delta(U)$ means that it moves the model toward Bob.

We use two studies because complete group ablations and prediction over ordinary subset masks answer different questions. The direct study uses 21 selected group masks, 128 matched name pairs, and eight prompt conditions formed by four lexical frames in the standard ABBA and BABA orders. These two orders reverse how the names are introduced relative to the person who later repeats as the subject. We repeat the study with two ablation conventions. Template-conditioned mean ablation replaces a selected head output with its average clean activation on that prompt template; zero ablation replaces it with zero. Agreement across the two reduces the chance that the interaction ordering comes from one replacement convention. This study asks what happens at the corners of the intervention space, where whole groups are removed.

The predictive study measures 240 masks over 256 name pairs on one prompt template. A masked head is replaced by its average clean activation on prompts from that template. We score the 239 non-clean masks using five folds, while retaining the clean mask as an anchor in every training fold. This study asks which measurement model best predicts unseen subset interventions drawn from the tested mask distribution.

\paragraph{Direct group ablations reproduce conditional backup.}

For two groups $A$ and $B$, define
\[
I_{AB}=\Delta(A+B)-\Delta(A)-\Delta(B).
\]
If their effects add independently, $I_{AB}=0$. A positive value means that removing both groups causes more damage than we would predict by adding their separate effects.

The primary and backup groups give a concrete example. On the original template, removing $P$ lowers the Bob--Alice margin by $0.463$, while removing $B$ lowers it by $0.350$. An additive model therefore predicts that removing both groups will lower the margin by $0.813$. The measured reduction is instead $1.869$, giving $I_{PB}=1.055$.

In terms of the prompt, the backup path looks relatively unimportant while the primary heads are still writing Bob. Once the primary path is gone, losing the backups becomes much more damaging. For illustration, suppose the clean model favors Bob by $1.5$ logits. The additive estimate predicts that Bob will remain ahead by $0.687$ after both groups are removed. The measured joint effect instead gives a margin of $-0.369$, placing Alice ahead of Bob. The value $1.5$ is illustrative, but it shows what the interaction means for an actual prediction: an additive observer can put the model on the wrong side of the Bob--Alice decision boundary.

The Negative Name Movers produce a different sign. Averaged over the eight lexical/order conditions, mean ablation gives $\Delta(E)=-2.270$ and $\Delta(P+E)=-0.136$; zero ablation gives $-2.335$ and $-0.769$. Removing $E$ alone therefore improves the IO-versus-subject margin. This is consistent with copy suppression: a mechanism that improves calibration over the broader model distribution can still oppose a narrow task-specific logit difference.

\begin{table}[t]
\centering
\caption{Direct IOI group interactions across four lexical frames in ABBA and BABA order, with 128 shared matched name pairs. Intervals use a crossed bootstrap that resamples lexical frames and one shared name-pair vector while retaining both orders. The last column divides by the number of possible cross-group head pairs; it does not assign the same effect to every individual pair.}
\label{tab:ioi-direct}
\scriptsize
\setlength{\tabcolsep}{2pt}
\begin{tabular}{@{}lccc@{}}
\toprule
Pair (heads) & \shortstack{mean ablation\\$I$ [95\%]} & \shortstack{zero ablation\\$I$ [95\%]} & \shortstack{$I$/pair\\(mean/zero)} \\
\midrule
$P\times B$ (24) & 1.080 [.912, 1.262] & 1.268 [1.013, 1.558] & 0.045/0.053 \\
$P\times E$ (6)  & 1.830 [1.569, 2.076] & 1.681 [1.474, 1.889] & 0.305/0.28 \\
$B\times E$ (16) & .502 [0.429, 0.576]    & 0.518 [0.416, 0.618]    & 0.031/0.032 \\
\bottomrule
\end{tabular}
\end{table}

\paragraph{The primary--negative interaction is larger.}

The $P\times E$ result also has a direct interpretation in the Alice--Bob example. Under mean ablation, removing the primary heads lowers the Bob--Alice margin by about $0.304$. Removing the Negative Name Movers while leaving the primary path intact raises the margin by $2.270$. If these effects were independent, removing both groups would raise the margin by
\[
2.270-0.304=1.966.
\]
The observed increase is only $0.136$.

The reason is conditional copy suppression. While the primary heads are writing Bob, the Negative Name Movers have a strong copied-name signal to act against. Removing the negative heads then releases that suppression and gives Bob a large boost. Once the primary heads have also been removed, much less Bob evidence remains for the negative heads to suppress. Removing them therefore provides almost no boost. The difference between the predicted and observed responses, $1.830$ logits, is the measured $P\times E$ interaction.

Every tested lexical/order condition gives the same ordering, $PE>PB>BE$, under both ablation conventions. The $PE-PB$ contrast is $0.750$ [$0.562,0.934$] under mean ablation and $0.413$ [$0.207,0.611$] under zero ablation. Both contrasts remain positive when we remove any one lexical frame.

The difference is not an artifact of group size. After dividing by the number of possible cross-group head pairs, the normalized $PE$ effect is $6.8\times$ the normalized $PB$ effect under mean ablation and $5.3\times$ under zero ablation.

These pair terms are not a complete circuit decomposition. A three-way residual of $0.319$ [$0.215,0.439$] remains under mean ablation, and $0.601$ [$0.417,0.790$] remains under zero ablation. Table~\ref{tab:ioi-direct} therefore ranks three documented group interactions. It does not assign all circuit behavior to pairs or claim that every individual $P$--$E$ head pair has the same effect.

\paragraph{Held-out subset prediction gives the same ordering.}

The predictive study asks whether these interactions help predict unseen subset masks, rather than only complete group ablations. We stratify masks by the exact number of removed primary heads, a binned count of removed Backup Name Movers, and the exact number of removed Negative Name Movers. We oversample masks with at least two removed primary heads because conditional backup is most visible after the primary path has been weakened.

A count-additive baseline gives each group its own flexible dose response. It can learn, for example, that removing four backup heads has more than twice the effect of removing two. What it cannot express is that removing a Negative Name Mover helps Bob while the primary writers remain but helps much less after those writers have been removed.

We represent that dependence by adding one normalized count product at a time:
\[
q_Pq_B,\qquad q_Pq_E,\qquad q_Bq_E,
\]
where $q_P=n_P/3$, $q_B=n_B/8$, and $q_E=n_E/2$. These are interactions between group-level ablation counts, not between identified pairs of individual heads.

The count-additive baseline has MAE $.386$ and $R^2=0.752$. A per-head additive check gives similar values: MAE $0.379$ and $R^2=0.759$. Adding the $PB$, $PE$, or $BE$ term gives MAE $0.377$, $0.259$, and $0.386$, respectively, and $R^2=0.774$, $0.887$, and $0.754$.

Across 2,000 paired prompt bootstraps with fixed five-fold splits, $PB$ improves MAE by $0.0083$ (95\% CI $[0.0064,0.0101]$), while $PE$ improves it by $0.1265$ (95\% CI $[0.1195,0.1338]$). The $BE$ change is $-0.0002$ (95\% CI $[-0.0015,0.0011]$), which is indistinguishable from no improvement. Adding all three terms gives MAE $0.236$, $R^2=0.907$, and an MAE improvement of $0.1499$ (95\% CI $[0.1425,0.1578]$).

The ordering also survives changes to the fold assignment. The $PE$ gain is positive in all 200 audited assignments, $PB$ in 199, and $BE$ in 154. Among the three tested single interaction terms, $PE$ is therefore both the largest and the most stable. $PB$ gives a smaller but repeatable gain, while $BE$ is split-sensitive.

\begin{figure}[t]
\centering
\begin{subfigure}{0.48\linewidth}
\includegraphics[width=\linewidth]{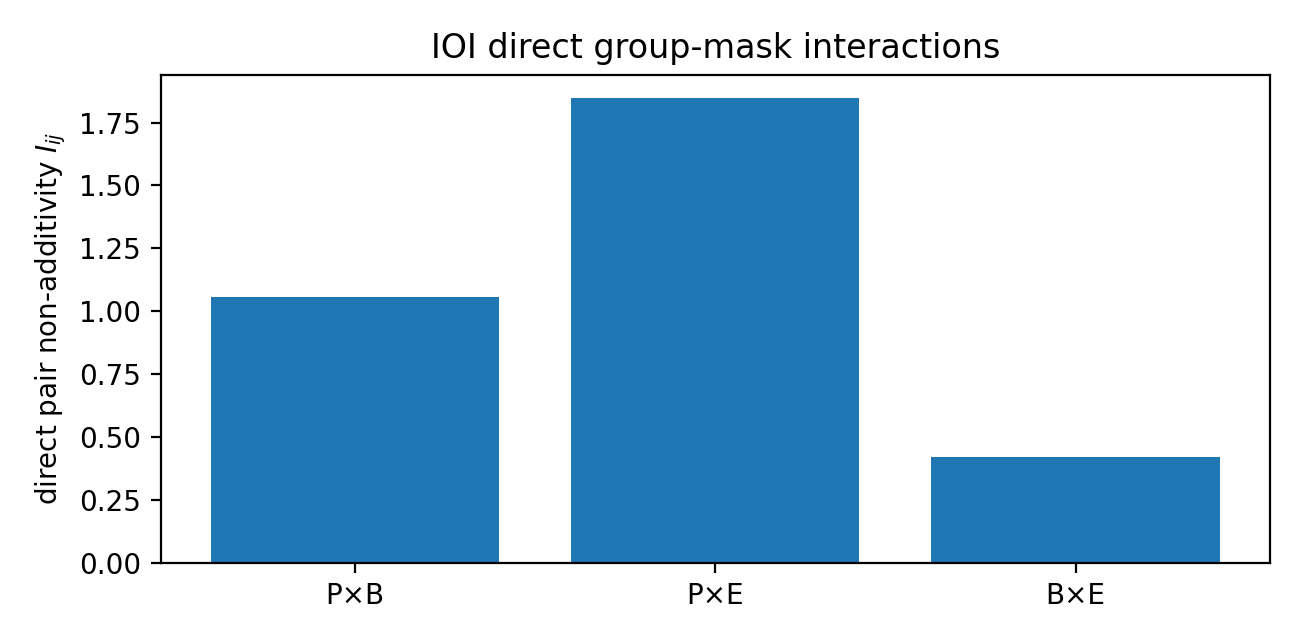}
\caption{Direct group-mask interactions}
\end{subfigure}\hfill
\begin{subfigure}{0.48\linewidth}
\includegraphics[width=\linewidth]{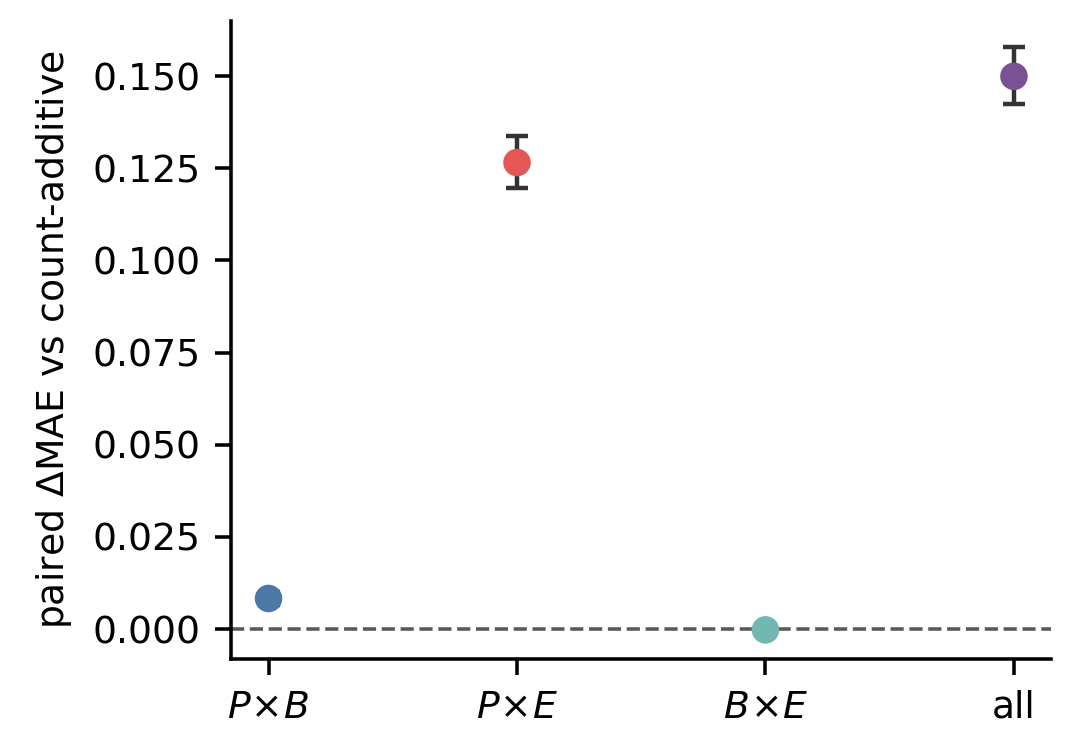}
\caption{Held-out prediction gain}
\end{subfigure}
\caption{IOI interaction measurements. Left: direct group interactions on the original-template mean-ablation surface. Right: held-out gains from adding one count-product term to the count-additive model. $P\times E$ gives the largest of the three tested gains; the interval for $B\times E$ crosses zero.}
\label{fig:ioi-anchor}
\end{figure}

\paragraph{What mechanistic tomography adds.}

Prior IOI work supplies the candidate head groups and their qualitative roles. Mechanistic tomography supplies the procedure for testing whether those groups are sufficient for the intended prediction task. We first fit the additive family and test it on held-out interventions. Its residual shows that separate group effects cannot predict the finite response, so we lift the design with the three group-pair terms. The $P\times E$ column expresses that removing Negative Name Movers has one effect while the primary path remains and another after it has been weakened; adding it reduces held-out MAE from $0.386$ to $0.259$. MT does not discover the IOI groups. It identifies when the additive family has reached its limit and which relations among the documented groups are needed to predict finite interventions under the declared distribution.

\paragraph{What the two measurements tell us.}

The direct and predictive studies agree on the ordering of the three tested interactions, but they should not be treated as interchangeable. A complete $P+B$ ablation measures the corner where both groups are gone. A predictive term measures how useful that interaction is across the subset masks the observer will encounter. A strong corner interaction can make only a small average predictive contribution if the measurement distribution rarely approaches that corner.

This distinction explains why the known $P\times B$ backup response can be clear under complete group ablation yet add much less predictive value than $P\times E$ over the tested subset distribution. It also states the limit of the claim. We find that $P\times E$ is the strongest and most stable of three group-level interaction terms under the tested design. We do not identify a particular head pair as the dominant circuit edge, and the predictive comparison remains limited to one prompt template and mean ablation.

\subsection{Finite calibration is sufficient on the tested Qwen-2.5-7B surface}\label{sec:qwen}

The preceding experiments use exact posteriors, planted effects, compiler labels, or documented circuit groups. We next ask whether the operational procedure can still be used when none of those references is available. This is a weaker test: it cannot tell us whether an observer has recovered the model's mechanism. It can still tell us whether the observer predicts the finite responses it was built to support.

We fix Qwen-2.5-7B-Instruct~\cite{qwen2025report} and a refusal-margin readout. The readout compares the model's log probability of four refusal stems with its log probability of four compliance stems at the first generated position. Harmful prompts come from HarmBench and safe, superficially safety-related prompts come from XSTest~\cite{mazeika2024harmbench,rottger2024xstest}. We use these prompts to obtain a varied refusal surface, not to evaluate the model's safety. The primary endpoint has no generated continuation and no LLM judge.

The intervention basis contains eight layerwise harmful-versus-benign contrast directions, measured at layers 5, 8, 11, 14, 17, 20, 23, and 26 at the last prompt token. We hold this basis, the prompt splits, and the action library fixed. The design contains 401 actions at four mask densities. Calibration and validation use scales $0.5$ and $1.0$; the final test uses the held-out scale $0.75$. The test set crosses 128 held-out actions with 224 held-out prompts.

We compare two response maps. The calibrated additive map contains the declared first-order action features and scale corrections. The lifted map adds all 28 products between the eight layer coefficients. Both models choose their ridge penalty on the same validation split. We then evaluate both on the same held-out action--prompt pairs.

\begin{table}[t]
\centering
\caption{Held-out prediction on the Qwen-2.5-7B refusal-response surface. The test contains 128 actions at the unseen scale $0.75$ and 224 prompts.}
\label{tab:qwen}
\scriptsize
\setlength{\tabcolsep}{5pt}
\begin{tabular}{lrrrr}
\toprule
Response map & parameters & MAE & RMSE & $R^2$ \\
\midrule
Calibrated additive & 20 & .003790 & .004699 & .9829 \\
Lifted pairwise & 48 & .003801 & .004613 & .9835 \\
\bottomrule
\end{tabular}
\end{table}

The additive map reaches held-out $R^2=.9829$ and MAE $.003790$. The lifted map reaches $R^2=.9835$ and MAE $.003801$. The estimated relative MAE improvement from lifting is therefore $-0.29\%$. A paired prompt-family-by-action bootstrap gives a 95\% interval of $[-3.56\%,5.65\%]$.

The result is a stopping decision, not an equivalence claim. The interval includes zero, so we detect no lifted advantage. It also extends slightly above the predeclared 5\% practical threshold, so the experiment does not rule out every meaningful pairwise improvement. Within this basis, prompt distribution, and intervention regime, however, the measured residual gives no reason to replace the calibrated additive map with the larger one.

This result complements the IOI study. In IOI, the additive residual points to a specific missing relation and the $P\times E$ term materially improves held-out prediction. On the Qwen surface, finite calibration already makes the simpler map adequate. Model size alone therefore does not decide which measurement order is needed; the declared basis and held-out response do. Because the Qwen endpoint is behavioral, this experiment establishes transfer of the measurement procedure to a modern open-weight model, not recovery of a ground-truth mechanism or a scaling law.\footnote{The public artifact repository at \url{https://github.com/kwisatzh/mechanistic-tomography} contains the complete experiment, frozen measurements, source and environment fingerprints, and a Colab notebook that reproduces Table~\ref{tab:qwen} without loading the model. The optional GPU path reruns the finite measurements from the pinned model revision.}

\FloatBarrier
\section{Related work}

\textbf{Ground-truthed belief-state laboratories.} Bayesian wind tunnels train transformers on probabilistic systems whose exact posterior is available, then study the geometry and learning dynamics of the learned representations~\cite{agarwal2025geometry,agarwal2025gradient,agarwal2025scaling}. HMM-trained transformers can encode belief states with nontrivial residual-stream geometry~\cite{shai2024belief}. Bernstein--von Mises theory describes asymptotic Gaussian posteriors for fixed latent parameters~\cite{vaart1998asymptotic}; our task instead filters a hidden state that changes over time. We use the wind-tunnel setup because the exact filtered posterior supplies an external reference for the observer. Mechanistic tomography adds the downstream question: what happens when that estimate enters a fixed control loop, and can target control hide movement of a nuisance state?

\textbf{Methods that obtain effect measurements.} Activation patching and causal tracing use finite internal interventions to localize components, with conclusions that can depend on the corruption and readout metric~\cite{meng2022rome,zhang2024patching}. Attribution patching obtains a local first-order map from gradients~\cite{syed2024attrib}; integrated Hessians and HVP correction add interaction or curvature information~\cite{janizek2020integrated,zhang2026lies}. SHAP and Faith-Shap use coalition measurements to assign main and higher-order feature effects~\cite{lundberg2017shap,tsai2023faithshap}. Compressed-sensing localization, SPEX, and ProxySPEX recover sparse components or interactions from aggregate masks~\cite{bair2026compressed,kang2025spex,butler2025proxyspex}, while noisy finite differences supply the classical scale tradeoff~\cite{shi2022finite}. Mechanistic tomography does not claim that these methods estimate the same quantity. It identifies their shared measurement structure and adds a finite-response residual, a calibration stopping rule, a scale--density design bound, basis-relative writeback, and observer validation.

\textbf{Methods that test causal meaning.} Causal abstraction asks whether interventions on a model correspond to interventions on a proposed higher-level variable, and DAS searches for representations that realize such variables~\cite{geiger2024das,geiger2025causal}. Probe controls distinguish information that is merely decodable from information tied to an intended role~\cite{hewitt2019probes}. MIB shows empirically that mechanistic conclusions depend on the target, ablation convention, and causal metric~\cite{mueller2025mib}. These methods ask whether a proposed representation has the claimed meaning. Mechanistic tomography addresses the preceding measurement problem: which interventions identify the proposed map, what error remains, and whether the estimate predicts held-out interventions. Its Tracr writeback and control tests then supply use-specific evidence beyond fit quality.

\textbf{Methods that evaluate action and control.} Steering benchmarks measure target effects, collateral changes, coherence, and reliability~\cite{bhalla2025unifying,wu2025axbench,tan2024steering}. Orgad et al. argue that interpretability should be judged by the actions it enables and the evidence validating those actions~\cite{orgad2026actionable}. PID Steering supplies a concrete feedback controller for activation edits~\cite{nguyen2025pid}. This work typically begins with a chosen steering signal. Mechanistic tomography asks how the state estimate that drives such a controller was measured, whether it predicts finite responses under the available access regime, and whether its error reaches target or nuisance behavior when the controller and actuator are held fixed.

\textbf{Tomography, representation engineering, and sparse recovery.} Linear artificial tomography (LAT), introduced in representation engineering and evaluated by AxBench, estimates a concept direction from contrastive activations~\cite{zou2023repeng,wu2025axbench}. Despite the shared word, mechanistic tomography targets component-effect and interaction maps from designed interventions rather than a contrastive concept direction. Its closer structural precedent is network tomography, which estimates hidden traffic quantities from boundary measurements and can use compressed sensing when the hidden object is sparse~\cite{vardi1996tomo,firooz2010cs}. Traffic-engineering work also shows that the most accurate traffic-matrix estimate need not produce the best routing decision~\cite{roughan2003te}. Mechanistic tomography brings both ideas into interpretability: design measurements for a hidden internal map, then report reconstruction and downstream action separately when the estimate becomes an observer.

\section{From effect maps to tested observers}\label{sec:discussion}

Mechanistic tomography gives researchers a workflow rather than a ranking of interpretability methods. Table~\ref{tab:designs} summarizes that workflow across access regimes; the four rules below state how to use it.

\begin{table}[t]
\centering
\caption{Observer selection as a measurement decision.}
\label{tab:designs}
\scriptsize
\setlength{\tabcolsep}{2pt}
\begin{tabular}{p{0.20\linewidth}p{0.21\linewidth}p{0.25\linewidth}p{0.25\linewidth}}
\toprule
Regime & Available access & Start with & Escalate when \\
\midrule
Additive, white-box & gradients and finite evaluations & AtP, then a few finite probes & held-out error does not reduce to a small gain correction \\
Additive, forward-only & scalar intervention responses & sparse subset measurements & the map is dense or held-out prediction fails \\
Interactional, white-box & gradients, HVP queries, and finite evaluations & AtP plus HVPs for local curvature & local support fails, or finite-scale probes show that calibration is needed \\
Interactional, forward-only & scalar intervention responses & lifted subset measurements & lifted coverage or conditioning fails \\
Uncertain basis & candidate coordinates and interventions & simplest declared basis & residuals persist or writeback fails \\
Downstream use & fixed controller and actuator & lowest-dimensional validated observer & target, collateral, or transfer tests fail \\
\bottomrule
\end{tabular}
\end{table}
\FloatBarrier

\begin{enumerate}[leftmargin=*,label=\textbf{\arabic*.}]
\item \textbf{State the intended use before choosing the measurements.}

A study should declare its target, basis, access regime, intervention scale, context distribution, and downstream use before choosing measurements (Secs.~\ref{sec:formulation} and~\ref{sec:limits}).

\item \textbf{Start with the cheapest measurement family that can represent the target.}

With gradients, begin with the local map and use finite probes to test it; with forward-only access, aggregate measurements can recover a compressible finite-effect map (Secs.~\ref{sec:decisions}, \ref{sec:aggregate}, and~\ref{sec:step0}).

\item \textbf{Let held-out residuals choose the next measurement.}

A shared scale error suggests calibration, conditional error suggests interactions, and persistent error after lifting may require a different basis; Secs.~\ref{sec:interactions}--\ref{sec:qwen} show each branch and a case in which the simpler family is sufficient.

\item \textbf{Validate an observer through its intended action.}

Report reconstruction, finite-response prediction, and downstream control separately; Secs.~\ref{sec:control} and~\ref{sec:result-control} show why target success alone is insufficient.
\end{enumerate}

\noindent\textbf{Where the evidence remains narrow.}\par
\noindent The planted interaction experiment assumes sparsity and a basis containing the correct pair terms; natural mechanisms may violate either assumption. Tracr supplies privileged compiler labels. The IOI direct ordering survives eight lexical/order conditions and two ablation conventions, while its predictive comparison uses one prompt template and mean ablation. Qwen-2.5-7B extends the procedure to a larger instruction-tuned model, but its refusal-margin endpoint provides behavioral ground truth only. It does not show that the fitted map recovers a represented state or circuit, and one model, basis, and response surface cannot establish a scaling law. The scale--density result is also a conditional design bound rather than a measured $\eps\times\rho$ surface. The paper therefore establishes a procedure and several concrete regimes, not their prevalence across mechanistic interpretability.

\noindent\textbf{From the workflow to ObserverBench.}\par
\noindent The next question is how often this decision path branches in the same way across models, tasks, observers, and interventions. Isolated case studies cannot answer it if each chooses its own masks, held-out tests, controller, and success metric. ObserverBench is the natural next step~\cite{erramilli2026observerbench}: tasks declare the controlled model or system, intervention distribution, controller, actuator, and held-out metrics, while researchers supply the observer.

This separation supports adoption as well as scientific control. A researcher should not need to rebuild every model, circuit, and control loop to test a new observer. A shared interface can return the same prediction, control, collateral, and robustness measures for attribution maps, SAE features, learned probes, interaction-aware estimators, and new designs. Mechanistic tomography supplies the logic for deciding what must be measured and validated; ObserverBench tests how well different observers satisfy that contract as the available ground truth becomes weaker.

\clearpage
\appendix

\section{From one finite intervention to a recoverable map}\label{app:proof-measurement}

The model is nonlinear, but the tomography equation does not require it to be globally linear. It requires a more limited statement: around each clean activation, the response to a finite intervention can be separated into a local linear contribution and a residual. The proof performs that separation for one mask and then stacks the resulting measurements.

\begin{proof}

\noindent\textbf{Step 1: Treat the mask as one activation-space direction.}\par

Fix a context $c$. The mask $a$ assigns a coefficient to each component direction, so the complete activation displacement is
\begin{equation}
    u_c=D_ca.
\end{equation}
This notation turns an intervention that may touch many components into one direction $u_c$ through activation space. The scalar $\eps$ determines how far the intervention moves along that direction.

\noindent\textbf{Step 2: Separate the tangent response from curvature.}\par

Apply the second-order Taylor theorem along the path from $h_c$ to $h_c+\eps u_c$. There is a point $\xi_c$ on this segment such that
\begin{equation}
\begin{aligned}
    f_c(h_c+\eps u_c)-f_c(h_c)
    &=
    \eps\,\nabla f_c(h_c)^\top u_c\\
    &\quad+
    \frac{\eps^2}{2}
    u_c^\top\nabla^2 f_c(\xi_c)u_c.
\end{aligned}
\end{equation}

The first term is the response predicted by the tangent at the clean activation. The second is the part caused by curvature along the finite path. This is the only place where differentiability enters the reduction.

\noindent\textbf{Step 3: Express the tangent response as a measurement of component effects.}\par

Because $D_c=[d_1(c),\ldots,d_n(c)]$,
\begin{equation}
\begin{aligned}
    \nabla f_c(h_c)^\top D_ca
    &=
    \sum_{i=1}^n
    a_i\,\nabla f_c(h_c)^\top d_i(c).
\end{aligned}
\end{equation}
Each term is the local effect of one component direction, multiplied by the coefficient assigned to it by the mask.

The target map is
\begin{equation}
\begin{aligned}
    x_i
    &=
    \E_c\left[
    \left.
    \frac{\partial}{\partial\delta}
    f_c(h_c+\delta d_i(c))
    \right|_{\delta=0}
    \right]\\
    &=
    \E_c\left[
    \nabla f_c(h_c)^\top d_i(c)
    \right].
\end{aligned}
\end{equation}
Averaging the tangent response over contexts therefore gives
\begin{equation}
    \E_c\left[
    \nabla f_c(h_c)^\top D_ca
    \right]
    =
    a^\top x.
\end{equation}

This is the key reduction. The finite intervention may change many components, but its first-order response is one linear measurement of their unknown effects. The mask supplies the weights in that measurement.

Divide the Taylor expansion by $\eps$ and average over contexts:
\begin{equation}
    y_\infty(a,\eps)
    =
    a^\top x+r(a,\eps),
\end{equation}
where
\begin{equation}
    r(a,\eps)
    =
    \frac{\eps}{2}
    \E_c\left[
    u_c^\top\nabla^2 f_c(\xi_c)u_c
    \right].
\end{equation}

The division by $\eps$ removes the intervention scale from the linear term. One power of $\eps$ remains in the curvature term, which is why nonlinear error grows linearly with the intervention scale after normalization.

Using $\|\nabla^2 f_c\|_{op}\le L$,
\begin{equation}
\begin{aligned}
    |r(a,\eps)|
    &\le
    \frac{\eps}{2}
    \E_c\left[
    \|\nabla^2 f_c(\xi_c)\|_{op}
    \|u_c\|_2^2
    \right]\\
    &\le
    \frac{L}{2}\eps
    \E_c\|D_ca\|_2^2.
\end{aligned}
\end{equation}

The bound depends on both the scale $\eps$ and the displacement produced by the mask. A dense mask need not have the same nonlinear error as a sparse mask, even when both use the same scalar intervention size.

\noindent\textbf{Step 4: Add the error from estimating the response.}\par

The population response is not normally available exactly. We estimate it from a finite batch and possibly from stochastic or finite-precision model evaluations. Write the empirical response as
\begin{equation}
    \widetilde y_B(a,\eps)
    =
    y_\infty(a,\eps)
    +\frac{e_B(a)}{\eps}
    +\zeta_B(a,\eps).
\end{equation}

The two empirical terms appear separately because they scale differently. The quantity $e_B(a)$ is error in the unnormalized response numerator. After we divide the response by $\eps$, this error becomes $e_B(a)/\eps$. It can therefore dominate when the intervention is very small.

The term $\zeta_B(a,\eps)$ represents sampling error that remains after normalization. Pairing the clean and intervened evaluations can cancel much of their shared variation. This term therefore need not grow as $1/\eps$. Pairing is part of the measurement design, not only a statistical convenience.

Combining the curvature and empirical terms gives
\begin{equation}
    \widetilde y_B(a,\eps)
    =
    a^\top x+w(a,\eps),
\end{equation}
with
\begin{equation}
    w(a,\eps)
    =
    r(a,\eps)
    +\frac{e_B(a)}{\eps}
    +\zeta_B(a,\eps).
\end{equation}
Hence
\begin{equation}
    |w(a,\eps)|
    \le
    \frac{L}{2}\eps\E_c\|D_ca\|_2^2
    +\frac{\tau_B(a)}{\eps}
    +\nu_B(a).
\end{equation}

\noindent\textbf{Step 5: Stack the measurements.}\par

Repeat the construction for masks $a_1,\ldots,a_m$. Put each row $a_j^\top$ into the matrix $A$, each empirical response into $\widetilde y$, and each residual into $w$. The $m$ scalar equations then become
\begin{equation}
    \widetilde y=Ax+w.
\end{equation}

\end{proof}

The proof gives the observation model used by mechanistic tomography. Each finite intervention gives one equation about the local component-effect map. The residual states what that equation misses and shows how the miss depends on the intervention.

The lemma alone does not guarantee recovery. The rows of $A$ must still distinguish the possible maps. Curvature and sampling error may also depend on the mask, so $w$ is not independent noise added after the design has been chosen.

\subsection{When several measurements identify the map}

Suppose two components are always changed together. Their columns in $A$ are then identical. We can measure their combined effect, but no algorithm can tell which component caused it. Recovery therefore depends on how the masks vary, not only on how many masks we collect.

If only $k$ of the $n$ component effects matter, aggregate masks can reduce the number of forward interventions. Under the standard sparse-separation conditions, basis-pursuit denoising solves
\begin{equation}
    \min_z \|z\|_1
    \quad \text{subject to}\quad
    \|Az-\widetilde y\|_2\le \eta,
\end{equation}
where $\eta\geq\|w\|_2$. Classical sparse recovery then gives
\begin{equation}
    \|\widehat x-x\|_2
    \leq C_0\eta+C_1\frac{\sigma_k(x)_1}{\sqrt{k}}.
\end{equation}
The first term is the price of measurement error. The second is the price of treating an approximately sparse map as exactly sparse. With suitably normalized random designs, $m\gtrsim k\log(n/k)$ measurements are enough with high probability.

For example, suppose only three of one hundred heads have substantial effects. Each aggregate mask measures a different weighted sum of all one hundred effects. If the masks separate the possible three-head supports, the equations can reveal which heads matter and estimate their effects without patching every head separately.

Coordinate patching has a simple forward-only lower bound. With fewer than $n$ coordinate measurements, at least one coordinate is never observed. The zero map and a map supported only on that coordinate produce the same data, so no method can distinguish them in the worst case. This does not say that every transformer needs all $n$ patches; a useful prior can improve average performance.

Attribution patching is the white-box exception. One backward pass returns all coordinates of the local first-order map, so there is no inverse problem to solve at that stage. Finite measurements remain useful for a different reason: they test whether the local map predicts interventions at the scale where it will be used.

These recovery results are classical. Mechanistic tomography supplies the measurement equation and makes the access assumptions explicit. Sparse forward recovery requires a suitable basis and a design that separates sparse alternatives. Gradient access returns the local map directly. Neither route removes the need to test finite responses on held-out interventions.

\section{When calibration is enough}\label{app:calibration}

Proposition~\ref{prop:caldim} asks a practical question: if a local observer misses the finite response, can a few numbers repair it, or does the observer need new features?

Let \(x_{\mathrm{grad}}\) be the fixed first-order baseline. Suppose it ranks the important heads correctly but underestimates every finite effect by about thirty percent. One gain may be enough:
\begin{equation}
    x_{\mathrm{finite}}(\eps)
    \approx
    (1+\theta_\eps)x_{\mathrm{grad}}.
\end{equation}
This is calibration dimension one. A two-group correction, such as separate gains for Name Movers and Negative Name Movers, has dimension two.

A gain cannot create a response that is absent from the baseline family. If an effect appears only when two components are changed together, the observer needs a pair feature. More calibration data for the same additive family cannot supply it. Held-out masks distinguish these cases: they show whether a correction transfers or merely fits correlated calibration measurements.

\begin{proof}[Proof of Proposition~\ref{prop:caldim}]
The proof separates two cases. First, the finite response lies inside the declared correction family. Second, it does not.

\noindent\textbf{Step 1: Write calibration as a small linear system.}\par

Let \(a_1,\ldots,a_m\) be the calibration masks. Put the feature row \(\phi(a_j)^\top\) into row \(j\) of \(\Phi\), and define
\[
    M=\Phi B.
\]
The columns of \(B\) are the \(r\) corrections we allow. Thus \(M\) records how those corrections appear under the chosen masks.

If the finite response belongs to this family, then for some \(\theta_\eps\),
\begin{equation}
    F_\eps(a_j)
    =
    \phi(a_j)^\top(x_0+B\theta_\eps).
\end{equation}
After subtracting the fixed baseline response, the observations satisfy
\begin{equation}
    y-\Phi x_0=M\theta_\eps+\nu,
\end{equation}
where \(\nu\) contains measurement error. The unknown now has dimension \(r\), rather than the full dimension \(q\) of the effect map.

\noindent\textbf{Step 2: Use rank to decide whether the correction is unique.}\par

In the noiseless case, suppose two corrections produce the same measurements. Then
\[
    M(\theta_1-\theta_2)=0.
\]
If \(M\) has full column rank, its null space contains only zero. Hence \(\theta_1=\theta_2\). In the ideal case, an \(r\)-dimensional correction therefore needs \(r\) independent equations. Extra measurements improve conditioning and leave separate masks for validation.

\noindent\textbf{Step 3: Track measurement error.}\par

When \(M\) has full column rank, least squares gives
\begin{equation}
    \widehat\theta-\theta_\eps=M^\dagger\nu.
\end{equation}
Therefore
\begin{equation}
    \|\widehat\theta-\theta_\eps\|_2
    \leq
    \|M^\dagger\|_2\,\|\nu\|_2.
\end{equation}
Full rank makes the correction identifiable. The size of \(\|M^\dagger\|_2\) determines whether that recovery is stable.

\noindent\textbf{Step 4: State what calibration cannot remove.}\par

If the true response lies outside the declared family, every observer in that family has held-out error at least
\begin{equation}
    \inf_\theta
    \left\|
        F_\eps(\cdot)
        -
        \phi(\cdot)^\top(x_0+B\theta)
    \right\|_{L_2(\mathcal Q)}.
\end{equation}
This lower bound remains even with unlimited noiseless calibration data. More data can find the best member of the family, but it cannot make that family express a missing interaction or other missing feature.
\end{proof}

Calibration dimension is therefore a stopping rule. A low value says that the baseline has the right structure and needs only a small finite-scale correction. A persistent held-out residual says that the measurement family must change.

\section{How scale and mask density change the measurement}\label{app:scale-density}

Corollary~\ref{cor:scale-density} joins two choices that are often made separately. The scale \(\eps\) says how far each probe moves the activation. The density \(\rho\) says how many components it changes. Both affect measurement error, and density also affects whether the inverse problem can separate the unknown effects.

A large probe is easy to observe but travels farther from the local tangent. A very small probe stays local, but dividing by \(\eps\) magnifies any error already present in the response numerator. Dense masks may cover a sparse support quickly, but they can also cause larger activation changes or create correlated columns. The result below makes those tradeoffs explicit.

\begin{proof}[Proof of Corollary~\ref{cor:scale-density}]
We first bound the error in a response and then pass that error through the recovery method.

\noindent\textbf{Step 1: Bound curvature for one mask family.}\par

Lemma~\ref{lem:measurement-reduction} bounds the curvature term for one normalized mask by
\begin{equation}
    \frac{L}{2}\eps\,\E_c\|D_ca\|_2^2.
\end{equation}
For the family \(\mathcal A_\rho\), define
\[
    \kappa_\rho^2
    =
    \sup_{a\in\mathcal A_\rho}
    \E_c\|D_ca\|_2^2.
\]
The curvature error is then at most
\begin{equation}
    \frac{L}{2}\eps\kappa_\rho^2.
\end{equation}
It grows with \(\eps\) because a larger intervention travels farther from the local approximation.

\noindent\textbf{Step 2: Add sampling error.}\par

Uniformly over the same mask family, Lemma~\ref{lem:measurement-reduction} gives
\begin{equation}
    \eta_1(\eps,\rho)
    =
    \frac{L}{2}\eps\kappa_\rho^2
    +
    \frac{\tau_B(\rho)}{\eps}
    +
    \nu_B(\rho).
\end{equation}
The first term grows with scale. The second grows when scale becomes too small. The third is the normalized sampling error left after pairing or batching.

\noindent\textbf{Step 3: Account for the inverse solve.}\par

Let \(C_{\mathrm{rec}}(A_\rho)\) bound how much the recovery method can amplify measurement error. Then
\begin{equation}
\begin{aligned}
    \mathcal J_1(\eps,\rho)
    &=
    C_{\mathrm{rec}}(A_\rho)\eta_1(\eps,\rho)\\
    &=
    C_{\mathrm{rec}}(A_\rho)
    \left[
        \frac{L}{2}\eps\kappa_\rho^2
        +
        \frac{\tau_B(\rho)}{\eps}
        +
        \nu_B(\rho)
    \right].
\end{aligned}
\end{equation}
This step connects the transformer response to the inverse problem. A badly conditioned design can turn a small response error into a large error in the recovered map.

\noindent\textbf{Step 4: Choose scale after fixing density.}\par

Fix \(\rho\), assume \(\tau_B(\rho)>0\), and treat \(C_{\mathrm{rec}}\) and \(\nu_B\) as locally independent of \(\eps\). The scale-dependent terms are
\[
    \frac{L}{2}\eps\kappa_\rho^2
    +
    \frac{\tau_B(\rho)}{\eps}.
\]
Differentiating and setting the result to zero gives
\begin{equation}
    \eps^*(\rho)
    =
    \sqrt{
        \frac{2\tau_B(\rho)}
             {L\kappa_\rho^2}
    }.
\end{equation}
At this point, the bound balances moving too far from the tangent against dividing by a probe that is too small.

Density remains a separate design choice. It changes \(\kappa_\rho\), feature coverage, and \(C_{\mathrm{rec}}(A_\rho)\). The joint design must therefore consider both \(\eps\) and \(\rho\).
\end{proof}

The formula is a design guide, not a universal setting. Density also depends on the actuator normalization. Under fixed per-coordinate amplitude, changing ten heads usually moves the activation farther than changing one. Under fixed row norm, the same total perturbation is spread across the ten heads. A useful scale--density sweep should therefore report the normalization, response error, feature coverage, and conditioning together.

\section{Why first-order measurements miss a pure interaction}\label{app:blindness}

Proposition~\ref{prop:blindness} identifies a limit of the measurement family. If neither component has a first-order effect by itself, no amount of first-order data can reveal an effect that appears only when both components change. The problem is not a weak recovery algorithm; the required pair column is absent.

\begin{proof}[Proof of Proposition~\ref{prop:blindness}]

\noindent\textbf{Step 1: First-order observations.}\par

By assumption,
\[
    x_i=x_j=0.
\]
Gradients and singleton first-order measurements therefore report zero for both components. From these observations, a model with \(H_{0,ij}=0\) and one with \(H_{0,ij}\neq0\) look the same.

\noindent\textbf{Step 2: Change both components together.}\par

Let a mask assign coefficients \(a_i\) and \(a_j\) to the two directions. Before normalization, the second-order Taylor expansion contains
\begin{equation}
    \frac{1}{2}\eps^2
    \left(2a_i a_jH_{0,ij}\right)
    +
    O(\eps^3).
\end{equation}
This term vanishes when either component is absent and appears when both are present.

\noindent\textbf{Step 3: Add the missing measurement column.}\par

After division by \(\eps\), the cross term becomes
\begin{equation}
    a_i a_j\Gamma_{\eps,ij},
    \qquad
    \Gamma_{\eps,ij}
    =
    \eps H_{0,ij}+O(\eps^2).
\end{equation}
The product \(a_i a_j\) is the design column for the pair. In a quadratic response model, the relation is exact. A lifted design can recover the interaction if it separates this pair column from the other main and pair columns.
\end{proof}

More gradients cannot reveal a pure cross term when the relevant first-order coordinates are zero. The next measurement must change: use pair-aware forward interventions or measure curvature directly.

\section{When lifted measurements recover pair effects}\label{app:proof-lifted}

Proposition~\ref{prop:lifted} adds the smallest feature family that can represent a two-component interaction. Instead of asking an additive model to absorb the extra response, it gives each candidate pair its own coordinate.

Define
\begin{equation}
    \theta_\eps=(x_\eps,\mathrm{vec}\,\Gamma_\eps)\in\R^N
\end{equation}
and replace each mask row \(a\) by
\begin{equation}
    \phi(a)=\big[a_1,\ldots,a_n,\{a_i a_j\}_{i<j}\big].
\end{equation}

\begin{proof}[Proof of Proposition~\ref{prop:lifted}]
A finite response with main and pair effects has the form
\begin{equation}
\begin{aligned}
    \widetilde y(a)
    &=
    a^\top x_\eps
    +
    \sum_{i<j}a_i a_j\Gamma_{\eps,ij}
    +
    w_\eps(a)\\
    &=
    \phi(a)^\top\theta_\eps+w_\eps(a).
\end{aligned}
\end{equation}
Stacking the masks gives
\begin{equation}
    \widetilde y=\Phi(A)\theta_\eps+w_\eps.
\end{equation}
This is an ordinary linear inverse problem in a larger set of coordinates. If \(\theta_\eps\) is sparse or compressible, the residual is bounded, and \(\Phi(A)\) separates the relevant sparse alternatives, the standard stable-recovery bound from Appendix~\ref{app:proof-measurement} applies with \(A\) replaced by \(\Phi(A)\).
\end{proof}

For two components, the model is
\[
    \widetilde y(a)
    =
    a_1x_1+a_2x_2+a_1a_2\Gamma_{\eps,12}+w_\eps(a).
\]
Changing only component 1 measures \(x_1\). Changing only component 2 measures \(x_2\). Changing both reveals any extra response that the two singleton effects do not explain. That extra response is the pair term.

\subsection{Why the lifted design needs its own check}\label{app:lifted-conditioning}

A design that separates main effects may still fail after lifting. Suppose components \(i\) and \(j\) are always changed together, so \(a_i=a_j\in\{0,1\}\) on every training mask. Then
\[
    a_i a_j=a_i=a_j.
\]
The two main effects and the pair effect create identical columns. The fitted model may predict every training response, but the data cannot say which term caused it. Repeating the same mask pattern only repeats the ambiguity.

Dense independent Rademacher masks give a useful ideal reference. Their entries are independent \(+1\) or \(-1\) values, so main columns \(a_i\) and pair columns \(a_i a_j\) are orthogonal in the population. A finite design only approximates this property. Sparse masks, fixed mask sizes, and exclusions needed to preserve model behavior can make lifted columns highly correlated or identical.

The lifted matrix must therefore be checked directly. Singular values, restricted conditioning, and column correlations show whether the masks distinguish the proposed pairs. Held-out masks should also change the co-occurrence patterns used during fitting. If the recovered pair predicts those responses, it is less likely to be an arbitrary explanation of aliased columns.

Lifted tomography does not guarantee that pair effects are sparse or sufficient. It says exactly what extra columns are needed, what design must identify them, and how to test whether they transfer.

\section{What Hessian-vector products recover}\label{app:hvp}

Corollary~\ref{cor:hvp} gives the white-box route to interactions. A gradient returns the complete local first-order map. A Hessian-vector product (HVP) asks how that gradient changes along one chosen direction. Several designed directions can therefore measure a structured local curvature map without constructing the full Hessian.

\begin{proof}[Proof of Corollary~\ref{cor:hvp}]

\noindent\textbf{Step 1: Separate gradient and curvature information.}\par

A backward pass returns
\[
    \nabla f(h),
\]
which determines local first-order effects. It does not return the Hessian \(H_0\), which describes how those effects change as the activation moves.

\noindent\textbf{Step 2: Treat HVPs as measurements of curvature.}\par

For a chosen direction \(v\), an HVP returns
\[
    H_0v.
\]
This expression is linear in the unknown Hessian. Choosing directions \(v_1,\ldots,v_m\) therefore gives designed linear measurements of the local curvature map. Selected coordinates or projections of these vectors can be stacked into an inverse problem.

\noindent\textbf{Step 3: Recover structured pair effects.}\par

If the relevant entries of \(H_0\) are sparse or compressible and the HVP design separates the possible supports, standard sparse recovery applies. The recovered object is local curvature at the operating point.

\noindent\textbf{Step 4: Connect local curvature to a finite intervention.}\par

If the response is quadratic along the intervention path, the Hessian is constant and
\[
    \Gamma_\eps=\eps H_0.
\]
Multiplying the recovered curvature by \(\eps\) then gives the normalized finite pair map exactly.

For a general nonlinear path, the Hessian may change as the activation moves. One HVP at the starting point identifies only local curvature. Predicting the finite response then requires integrating curvature along the path or calibrating the local estimate with finite interventions.
\end{proof}

HVPs make local curvature cheap under white-box access. Lifted forward measurements instead estimate pair effects at the scale where the masks are applied. They agree when local curvature transfers to that finite scale; otherwise the remaining difference is again a finite-scale residual to measure.

\end{document}